\documentclass[Afour,sageh,times]{sagej}

\usepackage{moreverb,url}

\usepackage[colorlinks,bookmarksopen,bookmarksnumbered,citecolor=red,urlcolor=red]{hyperref}

\newcommand\BibTeX{{\rmfamily B\kern-.05em \textsc{i\kern-.025em b}\kern-.08em
T\kern-.1667em\lower.7ex\hbox{E}\kern-.125emX}}

\def\volumeyear{2023}

\usepackage{makecell}
\usepackage{graphicx}
\usepackage{subfig}
\usepackage{amsmath}
\usepackage{amsthm} 
\usepackage{amssymb}
\usepackage[font=footnotesize]{caption} 
\usepackage{subfig} 
\usepackage{color}
\usepackage{algorithm} 
\usepackage{algpseudocode} 
\usepackage{enumitem} 
\setlist[enumerate]{leftmargin=*,label={\arabic*)},itemsep=0pt,topsep=5pt,listparindent=15pt}
\usepackage{arydshln} 
\usepackage{multirow} 
\usepackage{ulem} 
\usepackage{bm}
\usepackage{tikz}
\usetikzlibrary{calc} 
\usetikzlibrary{shapes} 
\usetikzlibrary{chains}
\usetikzlibrary{fit}
\usetikzlibrary{arrows}
\usetikzlibrary{decorations.text} 
\usetikzlibrary{decorations.markings}
\usetikzlibrary{decorations.pathmorphing} 
\usetikzlibrary{shadows}
\usetikzlibrary{patterns}
\usetikzlibrary{matrix}
\usepackage{pgfplots}
\usepackage[europeanresistors]{circuitikz}
\usepackage[outline]{contour} 
\contourlength{1.5pt}

\newtheorem{theorem}{Theorem}

\newtheorem{proof}{Proof}[section]

\newcommand{\R}{\mathbb{R}}

\graphicspath{{figures/}}

\newcommand{\my}{}

\begin{document}

\runninghead{Zian Ning et al.}

\title{A Bearing-Angle Approach for Unknown Target Motion Analysis Based on Visual Measurements}

\author{Zian Ning\affilnum{1,2} and Yin Zhang\affilnum{2} and Jianan Li\affilnum{2} and Zhang Chen\affilnum{3}, and Shiyu Zhao\affilnum{4,2}}

\affiliation{\affilnum{1} Department of Computer Science \& Technology at Zhejiang University, Hangzhou, China\\
\affilnum{2} School of Engineering at Westlake University, Hangzhou, China\\
\affilnum{3} Department of Automation at Tsinghua University, Beijing, China\\
\affilnum{4} Research Center for Industries of the Future, Westlake University, Hangzhou, China
}

\corrauth{Shiyu Zhao, School of Engineering, Westlake University, No.600 Dunyu Road, HangZhou 310024, China.
\email{zhaoshiyu@westlake.edu.cn}}

\begin{abstract}
Vision-based estimation of the motion of a moving target is usually formulated as a \emph{bearing-only} estimation problem where the visual measurement is modeled as a bearing vector. Although the bearing-only approach has been studied for decades, a \emph{fundamental limitation} of this approach is that it requires extra lateral motion of the observer to enhance the target's observability. Unfortunately, the extra lateral motion conflicts with the desired motion of the observer in many tasks.
It is well-known that, once a target has been detected in an image, a bounding box that surrounds the target can be obtained.
Surprisingly, this common visual measurement especially its size information has not been well explored up to now.
In this paper, we propose a new \emph{bearing-angle} approach to estimate the motion of a target by modeling its image bounding box as bearing-angle measurements.
Both theoretical analysis and experimental results show that this approach can significantly enhance the observability \emph{without} relying on additional lateral motion of the observer.
The benefit of the bearing-angle approach comes with no additional cost because a bounding box is a standard output of object detection algorithms.
The approach simply exploits the information that has not been fully exploited in the past.
No additional sensing devices or special detection algorithms are required.
\end{abstract}


\maketitle

\section{Introduction}

This paper studies the problem of estimating the motion of a moving target object using a moving monocular camera. The target's geometric information such as its physical size is \emph{unknown} in advance. This problem is important in many fields \citep{Qiu2019, Griffin2021, Tekin2018}.
Our present work is particularly motivated by the task of aerial target pursuit, where a micro aerial vehicle (MAV) uses its onboard camera to detect, localize, and then pursue another flying MAV.
The task of aerial target pursuit, originally motivated by the interesting bird-catching-bird behaviors in nature \citep{Brighton2019}, potentially provides an effective approach to the defense of misused MAV \citep{Rothe2019, Dressel2019, Vrba2020}.

\begin{figure}
	\centering
	\includegraphics[width=\linewidth]{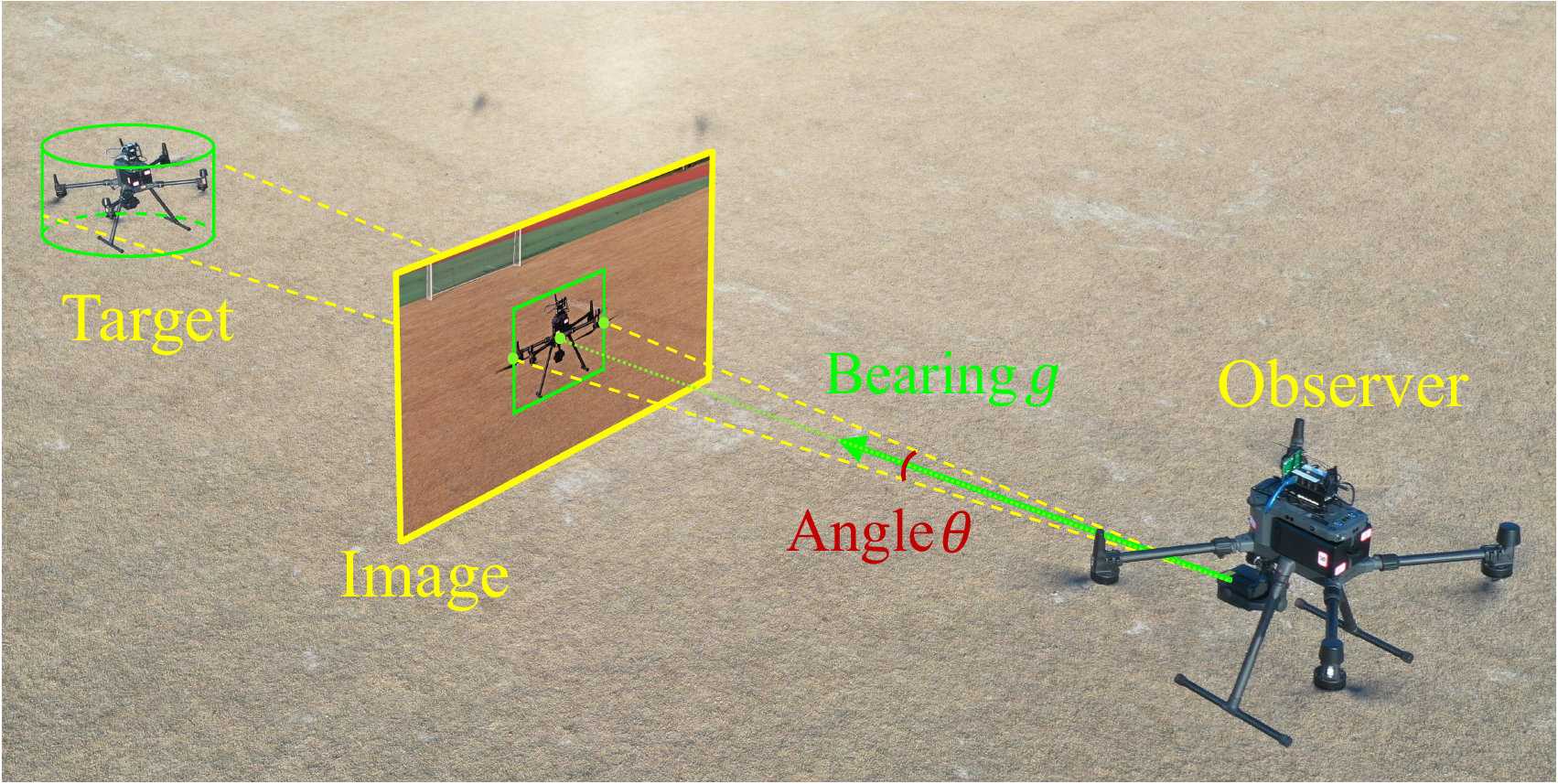}
	\caption{An observer MAV observes a target MAV with a monocular camera. The bearing $g$ and angle $\theta$ can be obtained from the bounding box that surrounds the target in the image.}
	\label{fig_architecture_outdoor}
\end{figure}

When a target has been detected in an image by a vision detection algorithm, we usually obtain a \emph{bounding box} that surrounds the target's image (see Fig.~\ref{fig_architecture_outdoor}).
The bounding box carries two types of useful information that can be used to estimate the target's motion.

The first type of useful information is the \emph{center point} of the bounding box.
The pixel coordinate of the center point can be used to calculate the spatial \emph{bearing vector} pointing from the camera to the target based on the pin-hole camera model \citep{Ma2012}.
Using the bearing vector to estimate the target's motion is referred to as \emph{bearing-only} target motion estimation \citep{Fogel1988, He2019, Li2022}.
As a problem that has been studied for more than 40 years, bearing-only target motion estimation was originally studied to estimate the motion of ships on the ocean surface \citep{hoelzer1978modified}, and regained increasing research attention in recent years in vision-based target estimation tasks \citep{Ponda2009, Anjaly2018, He2019}.

Bearing-only target motion estimation requires an \emph{observability condition}: The observer must have higher-order motion than the target and, more importantly, the higher-order motion must contain components that are orthogonal to the target's bearing vector \citep{Fogel1988}.
Motivated by this observability condition, enormous works have studied how an observer should move to enhance the observability \citep{Hammel1989, Sabet2016, Anjaly2018, He2019}.
For instance, in our recent work \citep{Li2022}, we proposed a helical guidance law so that a MAV moves along a helical curve to optimize the observability in the 3D space.

A \emph{limitation} of the observability condition of the classic bearing-only approach is that the observer must move in the lateral directions that are orthogonal to the bearing vector of the target.
Such additional lateral motion is usually unfavorable because it may conflict with the desired motion of the observer in many tasks.
For example, in an aerial target pursuit task, the pursuer is desired to approach the target as fast as possible and then keep stationary relative to the target. Then, the additional lateral motion would conflict with the desired motion.
It is, therefore, important to study other ways that can enhance the observability while avoiding unfavorable lateral motion.

The second type of useful information of a bounding box is its \emph{size} (either width or height).
The size of a bounding box is jointly determined by several factors such as the target's distance, the target's physical size, and the orientation of the camera.
The target's physical size is usually unknown in many tasks, especially in those antagonistic ones such as aerial pursuit of misused MAVs.
As a result, the size of the bounding box cannot directly infer the target's distance.
Nevertheless, it carries valuable information for localizing the target.

Surprisingly, the size information of the bounding box has not been well explored so far.
The work that is closely relevant to ours is the state-of-the-art one in \citep{Griffin2021}, where the size of a bounding box is used to localize unknown target objects.
Although the approach in \citep{Griffin2021} is inspiring, it relies on two assumptions: The target objects are stationary and the camera can only translate without rotating.
It is still an open problem how to estimate a target's motion when the two assumptions are not valid.
Moreover, the theoretical role of the size of a bounding box in target motion estimation has not been fully understood so far.
Although the work in~\citep{Vrba2020} also utilizes the size of the bounding box to estimate the target's position, it is assumed that the target's physical size is known in advance.

Estimating the motion of moving objects is also a fundamental problem in dynamic SLAM.
For example, the works in \citep{Yang2019,Qiu2019} firstly estimate the camera's pose and secondly estimate the target object's pose subject to a scale factor, and finally estimate the scale factor from multi-view measurements.
To estimate the target object's pose subject to a scale factor, \citep{Yang2019} and \citep{Qiu2019} rely on detecting, respectively, a 3D bounding box and sufficient feature points inside the 2D bounding box.
Different from \citep{Yang2019,Qiu2019}, our proposed approach merely utilizes a 2D image bounding box without further extracting feature points or a 3D bounding box inside the 2D bounding box.
As a result, one benefit is that this approach is more computationally efficient.
Moreover, this approach can handle the challenging small-target case where the target object is far and hence its image is small.
In this case, it would be unreliable to extract sufficient stable features or conduct 3D detection.

The aforementioned approaches in \citep{Griffin2021,Yang2019,Qiu2019} are all based on multiple views.
It is also possible to estimate the target's depth from a single view/image \citep{Tekin2018, Vrba2020}.
The single-view approach however requires prior information of the objects.
Moreover, it would be unable to successfully localize target objects with different sizes but similar appearances.
In this paper, we focus on the multi-view case.

In this paper, we propose a novel \emph{bearing-angle} target motion estimation approach that models a bounding box as bearing-angle measurements.
This approach can enhance the observability by fully exploiting the information in a bounding box rather than relying on the additional lateral motion of the observer.
The benefit of the proposed bearing-angle approach comes with no additional cost since the bounding box is a standard output of object detection algorithms.
The approach simply exploits the angle information that has not been fully exploited in the past.
No additional sensing devices or special detection algorithms are required.

The technical novelties of this approach are threefold.

1) The proposed approach does not directly use the size of a bounding box because the size is variant to the orientation of the camera.
That is, even if the target's relative position is unchanged, the size of the bounding box still varies when the camera rotates.
Motivated by this problem, we convert the size of the bounding box to an angle subtended by the target (see Fig.~\ref{fig_architecture_outdoor}).
The merit of using the angle measurement is that it is \emph{invariant} to the camera's orientation change (see Fig.~\ref{fig_cam_rotate}) and hence can greatly facilitate the estimator design.
In this way, the assumption in \citep{Griffin2021} that the camera can only translate but not rotate can be avoided.

2) Although the bearing-angle approach incorporates an additional angle measurement, it is nontrivial to see how to properly use this measurement because the angle does not directly infer the target's distance given that the target's size is unknown.
We notice that the angle is a joint nonlinear function of the target's physical size and relative distance.
Hence, the state vector, which only consists of the target's position and velocity in the conventional bearing-only approach, is augmented by the unknown target's physical size.
Since the bearing and angle measurements are all nonlinear functions of the target's state, we establish a pseudo-linear Kalman filter to properly utilize the measurements to enhance estimation stability.
Both simulation and real-world experiments verify the effectiveness of the proposed estimator.

3) Although an additional angle measurement is used, an additional unknown, the target's physical size, is also introduced into the estimator.
It is, therefore, nontrivial to see how the additional angle measurement can help improve the observability.
Motivated by this problem, we prove the necessary and sufficient observability condition for bearing-angle target motion estimation.
In particular, we show that the target's motion can be recovered if and only if the observer has a higher-order motion than the target.
Different from the bearing-only case, the higher-order motion is \emph{not} required to be in the lateral directions that are orthogonal to the bearing vector.
This is an important enhancement of the observability. As we show in various experiments, the bearing-angle approach can successfully recover the target's motion in many scenarios where the bearing-only approach fails.

\section{Related Work}

\subsection{Algorithms for bearing-only target motion estimation}
Bearing-only target motion analysis aims to estimate the target's motion states, such as position and velocity, using bearing measurement only.
It was originally motivated by ship localization and tracking in the ocean \citep{hoelzer1978modified}. With the rapid development of small-scale mobile robots equipped with cameras, the bearing-only approach regained increasing attention in recent years \citep{Ponda2009, Anjaly2018, He2019}.

Kalman filter-based estimators are widely used in the bearing-only target motion.
One challenge of applying the Kalman filter to the bearing-only estimation is the nonlinearity of the bearing measurement.
The conventional extended Kalman filter (EKF) exhibits divergence problems when applied to bearing-only target motion estimation \citep{Aidala1979, Lin2002}.
Several methods have been proposed to solve this problem.
They can be divided into two types.
The first type is the modified polar EKF, which was first proposed in \citep{hoelzer1978modified}.
In this approach, three observable quantities are separated from the unobservable ones to prevent divergence.
The work in \citep{Stallard1991} extends this approach to the case of spherical coordinates to track targets in 3D space.
The second type is the pseudo-linear Kalman filter, which is first proposed in \citep{Lingren1978} to solve the instability problem by transforming the nonlinear measurement equation into a pseudo-linear one.
However, this transformation makes the noise become non-Gaussian and highly correlated to the measurement matrix and then causes estimation bias.
Nevertheless, the work in \citep{Aidala1982} theoretically proves that the velocity estimation has no bias, and the position estimation bias can be removed by the observer's maneuvers.

Recently, other estimation algorithms based on advanced but more complex filters have been proposed.
The work in \citep{Farina1999} uses the maximum likelihood (MLE) algorithm to estimate the target's motion using bearing-only measurements.
The comparison with the Cramer-Rao lower bound indicates that the MLE-based estimator is effective against measurement errors.
The work in \citep{Dogancay2005} proposes a constrained total least-squares algorithm, which can improve the estimation accuracy when the error of bearing measurement is large.
Three different algorithms are used and compared in \citep{Lin2002}.
The results show that the EKF, the pseudo-linear filter, and the particle filter have similar performances in most situations, while the EKF loses track when the initial estimate error is large.

Another type of approach, called bearing-only trajectory triangulation \citep{Avidan2000}, estimates the target's position from the perspective of trajectory fitting.
It reconstructs the trajectory by intersecting parametric trajectory to a series of sight rays obtained from bearing measurement.
Once the trajectory is successfully fitted, the target's position at each time instant can be estimated by the intersection of the bearing and the trajectory.
The trajectory fitting relies on the assumption of the trajectory's shape.
However, in many applications, the target's trajectory is complex and unknown in advance.
Many consecutive studies aim to relax this assumption in various ways based on hypersurfaces \citep{Kaminski2004}, parametric temporal polynomials \citep{Yu2009}, or compact basis vectors \citep{Park2015}.

\subsection{Observability analysis of bearing-only target motion estimation}

Observability is a fundamental problem in bearing-only target motion estimation.
Early works mainly focus on whether the system is observable or not.
For example, the work in \citep{Lingren1978} uses the rank of observation matrix to determine the observability.
The work in \citep{Fogel1988} extends the observability criterion in \citep{Nardone1981} to the Nth-order target dynamics and inspires us for the observability analysis in Section~\ref{sec_observability_criteria}.
All these conditions indicate that the observer must have extra high-order motion in the lateral direction.
The observability condition can be significantly relaxed in our approach.

Unlike the works on determining whether the system is observable or not, some studies focus on quantifying the observability degree.
The work in \citep{Hammel1989} first introduces the Fisher information matrix (FIM) into the observability analysis.
The works in \citep{Sabet2016} and \citep{Anjaly2018} use FIM-based objective functions to maximize observability.
We also use the FIM in our former work \citep{Li2022} to optimize the 3D helical guidance law for better observability.
Another method called the geometric method uses the geometric relationship between the target and the observer in two consecutive time instants to derive the measure of observability \citep{He2019, Woffinden2009}, and the results are consistent with those derived using FIM.
Compared to the bearing-only approach, the observability degree of our bearing-angle method is sufficient to estimate the target's motion in many common scenarios such as tracking and guidance (see experiment results in Figs.~\ref{fig_matlab_3} and~\ref{fig_outdoor_1}).

\section{Problem Formulation}
\begin{figure}
\centering
\includegraphics[width=0.885\linewidth]{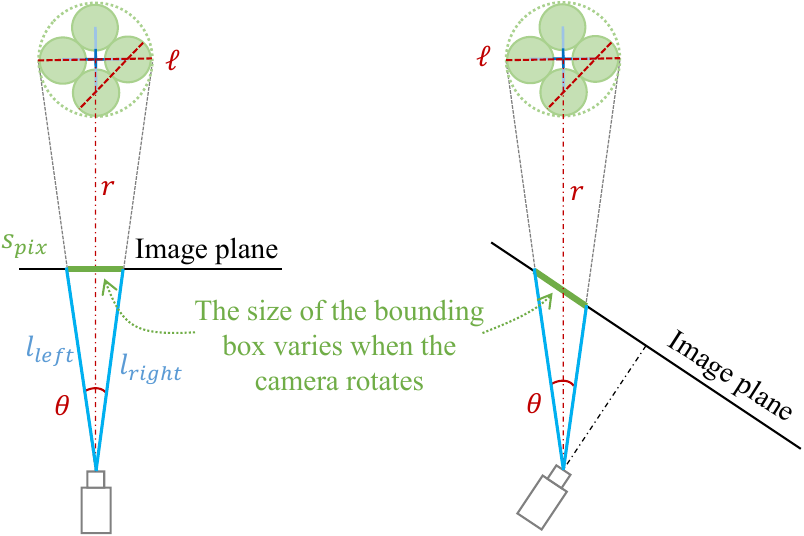}
\caption{The size of the bounding box varies when the camera rotates. By contrast, the angle subtended by the target object is invariant to the camera's orientation change.}
\label{fig_cam_rotate}
\end{figure}

Consider a target object moving in the 3D space. Its position and velocity at time $t_k$ are denoted as $p_T(t_k) \in\R^3$ and $v_T(t_k) \in\R^3$, respectively.
Suppose there is an observer carrying a monocular camera to observe the target.
The position of the observer is denoted as $p_o(t_k) \in\R^3$.
Here, we assume that the observer/camera's pose including its position and orientation can be obtained in other ways.
For example, it can be measured directly by RTK GPS~\citep{Li2022} or estimated by visual inertial odometry~\citep{Qiu2019}.
In the rest of the paper, the dependence of a variable on $t_k$ is dropped when the context is clear.

If the target object can be detected by a vision algorithm, we can obtain a bounding box surrounding the target object in the image.
Two types of information carried by the bounding box can be used to estimate the motion of the target.

First, the center point of the bounding box can be used to calculate the \emph{bearing} vector of the target.
In particular, denote $\my{g} \in \mathbb{R}^3$ as the unit bearing vector pointing from $p_o $ to $p_T $.
Suppose $\my{P}_\text{cam}\in\R^{3\times3}$ is the intrinsic parameter matrix of the camera \citep[Section~\ref{section_bearing-angle-target-motion-estimator}]{Ma2012}, and $\my{R}_\text{c}^\text{w} \in\R^{3\times 3}$ is the rotation from the camera frame to the world frame.
Then, the bearing vector $g$ can be calculated as
\begin{align*}
\my{g} =
\dfrac{
\my{R}_\text{c}^\text{w}
\my{P}_\text{cam}^{-1}
\my{q}_{\rm pix}
}{
\|\my{R}_\text{c}^\text{w}
\my{P}_\text{cam}^{-1}
\my{q}_{\rm pix}
\|
},
\end{align*}
where $\my{q}_{\rm pix} =[x_{\rm pix} , y_{\rm pix} , 1]^\mathrm{T} \in \mathbb{R}^3$.
Here, $(x_{\rm pix} ,y_{\rm pix})$ is the pixel coordinate of the center point of the bounding box.

Second, the size of the bounding box can be used to calculate the \emph{angle} subtended by the target in the camera's field of view.
The reason that we convert the bounding box's size to the angle is that the angle is invariant to the camera's orientation change (see Fig.~\ref{fig_cam_rotate}).
In particular, let $s_{\rm pix} $ denote the size of the bounding box.
It can be either the width or the height.
Let $\theta \in (0,\pi/2)$ be the angle.
According to the pin-hole camera model \citep[Section~\ref{section_bearing-angle-target-motion-estimator}]{Ma2012} and the law of cosine (see Fig.~\ref{fig_cam_rotate}), the angle can be calculated as
\begin{align*}
\theta = \arccos\left(\dfrac{l_\mathrm{left}^2 + l_\mathrm{right}^2 - s_\mathrm{pix}^2}{2l_\mathrm{left}l_\mathrm{right}}\right),
\end{align*}
where $l_\mathrm{left}=\sqrt{(f/\alpha)^2+(\delta x-s_\mathrm{pix}/2)^2+\delta y^2}\in\mathbb{R}$ and $l_\mathrm{right}=\sqrt{(f/\alpha)^2+(\delta x+s_\mathrm{pix}/2)^2+\delta y^2}\in\mathbb{R}$ are the distances in pixel from the camera center to the middle points of the left and right sides of the bounding box, respectively (Fig.~\ref{fig_architecture_outdoor}).
Moreover, $f$ and $\alpha$ denote the camera's focal length and single pixel size, respectively.
$i_{\text{width}}$ and $i_{\text{height}}$ represent the width and the height of the whole image in pixels, respectively. $\delta x=\|x_\text{pix}-i_\text{width}/2\|\in\mathbb{R}$ and $\delta y = \|y_\text{pix}-i_\text{height}/2\|\in\mathbb{R}$ are the distances between the center of the bounding box and the center of the image.

\begin{figure*}[!t]
	\centering
	\includegraphics[width=1\linewidth]{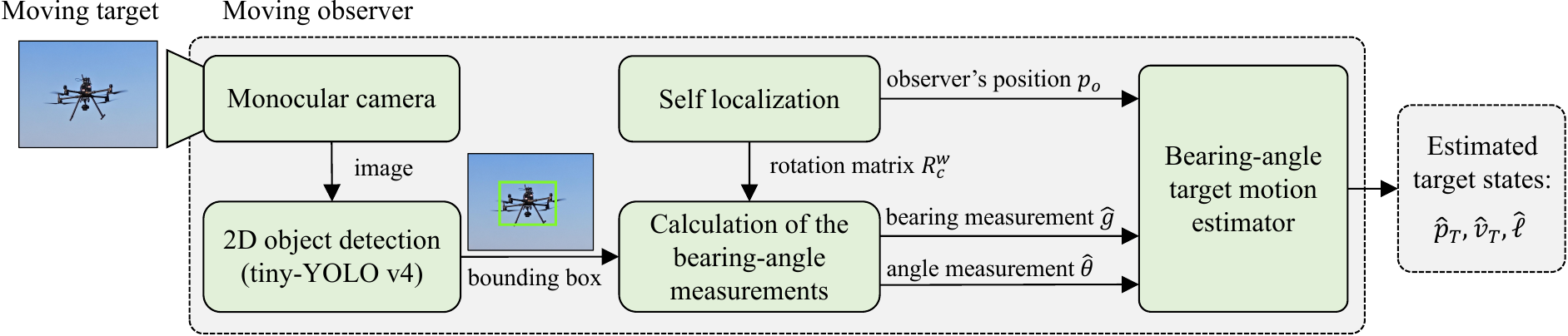}
	\caption{The architecture of the proposed approach. All the simulation and real-world experiments in this paper follow this architecture.}
	\label{fig_architecture_algorithm}
\end{figure*}

Our goal is to estimate the target's position and velocity, $p_T$ and $v_T$, based on the noisy measurements of the bearing vector $\my{g}$ and the angle $\theta$ together with the observer's own position $p_o$.
To achieve this goal, we propose a new bearing-angle target motion estimator (Fig.~\ref{fig_architecture_algorithm}).
The estimator is introduced in detail in Section~\ref{section_bearing-angle-target-motion-estimator}.
The observability of this estimator is analyzed based on Kalman's observability criterion in Section~\ref{sec_analysis_of_observability_matrix}.
We further prove a necessary and sufficient observability condition of the observer in Section~\ref{sec_observability_criteria}.
Numerical simulation results are given in Section~\ref{sec_matlab_simulation}.
More realistic AirSim simulation results are given in Section~\ref{sec_airsim_simulation}.
Finally, real-world experiments are given in Section~\ref{sec_real_world_experimental_validation}.

\section{Bearing-Angle Target Motion Estimator}
\label{section_bearing-angle-target-motion-estimator}

This section designs a bearing-angle target motion estimator based on the framework of pseudo-linear Kalman filtering. The key here is to establish appropriate measurement and state transition equations.

\subsection{States transition equation}
\label{sec_states_transition_equation}
The state vector of the target is designed as
\begin{align*}
\my{x}=
\left[
  \begin{array}{c}
    p_T \\
    v_T \\
    \ell \\
  \end{array}
\right]\in \mathbb{R}^7,
\end{align*}
where $\my{p}_T $ and $\my{v}_T $ are target's global position and velocity, respectively.
Here, $\ell>0$ is a scalar that represents the physical size of the target object in the dimension that is orthogonal to the bearing vector (Fig.~\ref{fig_cam_rotate}). In this paper, $\ell$ is assumed to be constant or varying slowly, which means that the physical size of the target object should be approximately invariant from different viewing angles. Here, $\ell$ corresponds to $\theta$, which further corresponds to either the width or height of the bounding box. Whether $\ell$ should correspond to the width or height depends on in which dimension the physical size of the target object is invariant when viewed from different angles.
More explanation is given in Section~\ref{sec_dynamical_model_of_size}.

Different from the bearing-only case where the state merely consists of the position and velocity, the state here is augmented by the target's physical size. This is due to the fact that the angle measurement is a function of the target's physical size, which should be estimated as well. One may wonder whether the state vector can also incorporate the target's acceleration. To estimate high-order motion (e.g., acceleration) of the target, the observer must have higher-order motion (e.g., nonzero jerk) according to the observability condition presented in Section~\ref{sec_observability_criteria}. Otherwise, the estimation would diverge. Therefore, it is preferred to exclude the acceleration and merely estimate the position and velocity.

If no information of the target's motion is available, it is common to model the target's motion as a discrete-time noise-driven double integrator:
\begin{align}\label{eq_state_transition}
	\my{x}(t_{k+1})=\my{F}\my{x}(t_k) +\my{q}(t_k) ,
\end{align}
where
\begin{align}
\label{eq_matrix_A}
\my{F}=
\begin{bmatrix}
\my{I}_{3\times3} & \delta t\my{I}_{3\times3} & \my{0}_{3\times1}  \\
\my{0}_{3\times3} & \my{I}_{3\times3}  & \my{0}_{3\times1}   \\
\my{0}_{1\times 3} & \my{0}_{1\times 3} & 1
\end{bmatrix}\in\mathbb{R}^{7\times 7},
\end{align}
with $\delta t$ as the sampling time, and $\my{I}$ and $\my{0}$ as the identity and zero matrices, respectively.
Here, $\my{q} \in\mathbb{R}^7$ is a zero-mean process noise satisfying $\my{q} \sim \mathcal{N}(0,\my{\Sigma}_q)$, where the covariance matrix is
\begin{align}
\my{\Sigma}_q=\text{diag}(0, 0, 0, \sigma_v^2, \sigma_v^2, \sigma_v^2, \sigma_\ell^2)\in\mathbb{R}^{7\times7}.
\end{align}
Here, $\sigma_v\in\mathbb{R}$ and $\sigma_\ell\in\mathbb{R}$ are the standard deviations of the target's velocity and size, respectively.
When the target's shape is irregular, $\ell$ may vary when viewed from different angles.
By letting $\sigma_\ell\ne0$, we can handle the case where $\ell$ varies slowly.
The dynamic modeling of $\ell$ is discussed in the following subsection.

\subsection{Dynamic modeling of target's physical size}
\label{sec_dynamical_model_of_size}

Since the target's physical size $\ell$ is a state variable to be estimated, it is important to discuss its dynamic model. In fact, the dynamic model of $\ell$ in \eqref{eq_state_transition} assumes that $\ell$ varies slowly. We next justify this modeling and provide more discussion.

First of all, $\ell$ corresponds to the physical size of the target object in the dimension that is orthogonal to the bearing vector. Its dynamics can be categorized into three cases.

\emph{1) $\ell$ is invariant.}
In theory, when $\ell$ is invariant, a change of $\theta$ implies a change of $r$.
As a result, the measurement of $\theta$ can help improve the system's observability, as proven in Section~\ref{sec_observability_criteria}.
An ideal case where $\ell$ is invariant is that the target object is a sphere or cylinder so that $\ell$ corresponds to its diameter \citep{Vrba2020}.
In practice, the target object does not have to be the ideal case. For example, consider an autonomous driving scenario where a focal vehicle uses a camera to localize its surrounding vehicles in the 2D plane.
Although the physical size of a surrounding vehicle changes greatly when viewed from behind or side, the height of the vehicle is \emph{invariant} from different side-view angles.
In this case, $\ell$ corresponds to the height of the vehicle, and we need to use the height of the image bounding box to calculate $\theta$.

\emph{2) $\ell$ varies slowly.}
If there does not exist any dimension in which the physical size of the target remains invariant, $\ell$ may vary slowly when the target is viewed from different angles. For example, in the tasks of aerial target pursuit, if the target is a quadcopter or hexacopter, then $\ell$ is approximately equal to the wheelbase but may vary slightly when viewed from different angles since the MAV is not a perfect cylinder.
In this case, $\ell$ corresponds to the wheelbase of the MAV, and we need to use the width of the image bounding box to calculate $\theta$.

If $\ell$ varies slowly, it can still be treated as invariant within short time intervals.
As long as the observability condition (Section~\ref{sec_observability_criteria}) is satisfied, the motion of the target as well as $\ell$ can be successfully estimated.
This fact is supported by the experimental results in Section~\ref{sec_sim_res_circular_scenario}.
It is however worth nothing that the performance of the proposed bearing-angle approach would degenerate to the conventional bearing-only one because the additional information brought by $\theta$ is used to estimate the time-varying $\ell$ rather than helping improve the system's observability.

\emph{3) $\ell$ varies rapidly.}
If $\ell$ varies rapidly due to certain reasons, it would be difficult to distinguish whether the change of $\theta$ is caused by the change of $\ell$ or the change of $r$.
For example, when a MAV is used to track a ground vehicle, $\ell$ in any dimension may vary rapidly when the relative motion between the MAV and the ground vehicle is highly dynamic.
In such scenarios, the additional information brought by $\theta$ is no longer sufficient to estimate the rapidly varying $\ell$ in this case. Additional visual information such as a 3D bounding box that indicates the target's 3D attitude is required. This is an important topic for future research but out of the scope of the present paper.

\subsection{Nonlinear measurement equations}

The bearing vector $\my{g} $ and the subtended angle $\theta $ are both nonlinear functions of the target's position. In particular,
\begin{subequations}
\label{eq_information}
\begin{align}
	\my{g} &=\dfrac{\my{p}_T -\my{p}_o }{r },
	\label{eq_bearing_measure} \\
	\theta &=2\arctan\left(\dfrac{\ell}{2r }\right)\approx \dfrac{\ell}{r },
	\label{eq_theta_measure}
\end{align}
\end{subequations}
where
$$r =\|\my{p}_T -\my{p}_o \|$$
is the distance between the target and the observer.
It is notable that there is an approximation in \eqref{eq_theta_measure}. This approximation is accurate.
Specifically, when $r>3\ell$, which is common in practice, it can be verified that the approximation error is less than $0.08\%$. The approximation error further decreases as $r$ increases.

In practice, measurements always contain noises.
First, denote $\hat{\my{g}} \in\mathbb{R}^3$ as the noise-corrupted bearing measurement. Then, we have
\begin{align}
\label{eq_noised_g_mear}
\hat{\my{g}}  = \my{R}\left(\my{\eta} , \epsilon \right) \my{g} ,
\end{align}
where $\my{R}\left(\my{\eta} , \epsilon \right) \in \mathbb{R}^{3\times 3}$ is a rotation matrix that perturbs $\my{g}$.
Here, $\my{\eta} \in\mathbb{R}^3$ is a unit vector representing a random rotation axis, and $\epsilon \in \mathbb{R}$ is a random rotation angle.
This rotation matrix would rotate the vector $\my{g} $ by an angle $\epsilon $ around the axis $\my{\eta} $.
The productive noise in \eqref{eq_noised_g_mear} can be transformed into an additive one:
\begin{align}\label{eq_noised_g_mear_add}
	\hat{\my{g}}  = \my{g}  + \my{\mu} ,
\end{align}
where $\my{\mu} =(\my{R}\left(\my{\eta} , \epsilon \right) - \my{I}_{3\times3})\my{g} \in\mathbb{R}^3$ is the measurement noise of the bearing vector.
The covariance of $\mu$ is derived in our previous work~\citep{Li2022}. Since the covariance is complex and involves unknown true values, we can approximately treat it as a Gaussian noise: $\mu\sim\mathcal{N}(0, \sigma_\mu^2 I_{3\times 3})$ \citep{Li2022}.

Substituting \eqref{eq_bearing_measure} into \eqref{eq_noised_g_mear_add} gives the \emph{nonlinear bearing measurement equation:}
\begin{align}\label{eq_bearing_measure_noise}
	\hat{\my{g}} &=\dfrac{\my{p}_T -\my{p}_o }{r } + \my{\mu} .
\end{align}

Second, denote $\hat{\theta} \in\mathbb{R}$ as the noise-corrupted measurement of the subtended angle. Then, we have
\begin{align}\label{eq_noise_theta}
	\hat{\theta} =\theta  + w ,
\end{align}
where $w \sim \mathcal{N}(0, \sigma^2_w)$ is the measurement noise.
Substituting \eqref{eq_theta_measure} into \eqref{eq_noise_theta} yields the \emph{nonlinear angle measurement equation:}
\begin{align}\label{eq_theta_measure_noise}
	\hat{\theta} &=\dfrac{\ell}{r } + w.
\end{align}

\subsection{Pseudo-linear measurement equations}

The measurement equations \eqref{eq_bearing_measure_noise} and \eqref{eq_theta_measure_noise} are nonlinear in the target's state. In the following, we convert the two equations to be pseudo-linear and then apply pseudo-linear Kalman filtering to achieve better estimation stability \citep{Lin2002}.

First, to convert the 3D bearing measurement to pseudo-linear, we introduce a useful orthogonal projection matrix:
\begin{align*}
	\my{P}_{\hat{\my{g}} }\doteq\my{I}_{3\times 3}-\hat{\my{g}} \hat{\my{g}}^\mathrm{T}  \in \mathbb{R}^{3\times 3}.
\end{align*}
This matrix plays an important role in the analysis of bearing-related estimation and control problems \citep{Zhao2019}. It has an important property: $$\my{P}_{\hat{\my{g}} }\hat{\my{g}} =\my{0}_{3\times 1}.$$
As a result, multiplying $r\my{P}_{\hat{\my{g}} }$ on both side of \eqref{eq_bearing_measure_noise} yields
\begin{align*}
\my{0}_{3\times 1}=\my{P}_{\hat{\my{g}} }(\my{p}_T -\my{p}_o) + r\my{P}_{\hat{\my{g}} }\my{\mu}
\end{align*}
and consequently
\begin{align*}
\my{P}_{\hat{\my{g}} }\my{p}_o =\my{P}_{\hat{\my{g}} }\my{p}_T  + r\my{P}_{\hat{\my{g}} }\my{\mu}.
\end{align*}
Rewriting this equation in terms of the target's state variables yields the \emph{pseudo-linear bearing measurement equation:}
\begin{align}\label{eq_pseudo_linear_measurement_g_equation}
\my{P}_{\hat{\my{g}} }\my{p}_o =
\begin{bmatrix}
\my{P}_{\hat{\my{g}} } &
\my{0}_{3\times4}
\end{bmatrix}
\left[
  \begin{array}{c}
    p_T \\
    v_T \\
    \ell \\
  \end{array}
\right]  +  r\my{P}_{\hat{\my{g}} }\my{\mu} .
\end{align}
Here, $\my{P}_{\hat{\my{g}} }\my{p}_o $ on the left-hand side is the new measurement, which is pseudo-linear in the target's state variables.
The reason that it is called "pseudo" is because the measurements also appear on the right-hand side of the equation, especially in the measurement matrix.

Second, we convert the nonlinear angle measurement in \eqref{eq_theta_measure_noise} to be pseudo-linear.
To that end, multiplying $r \my{\hat{g}} $ on both side of \eqref{eq_theta_measure_noise} yields
\begin{align}\label{eq_theta_pseudo_tem}
\hat{\theta} r\hat{\my{g}}  = \ell\hat{\my{g}} +wr\hat{\my{g}} .
\end{align}
It follows from \eqref{eq_bearing_measure_noise} that $r\my{\hat{g}}=\my{p}_T -\my{p}_o+r\mu$, substituting which into the left-hand side of \eqref{eq_theta_pseudo_tem} gives
\begin{align*}
\hat{\theta} (\my{p}_T -\my{p}_o+r\mu)  = \ell\hat{\my{g}} +wr\hat{\my{g}}.
\end{align*}
Reorganizing the above equation gives
\begin{align*}
\hat{\theta} \my{p}_o  = &\hat{\theta} \my{p}_T  - \ell\hat{\my{g}} + r(\hat{\theta}  \my{\mu}  - w \hat{\my{g}}).
\end{align*}
Rewriting this equation in terms of the target's state variables yields the \emph{pseudo-linear angle measurement equation:}
\begin{align}\label{eq_pseudo_linear_measurement_theta_equation}
\begin{aligned}
\hat{\theta} \my{p}_o  =&
\begin{bmatrix}
\hat{\theta} \my{I}_{3\times 3} & \my{0}_{3\times 3}  & -\hat{\my{g}}
\end{bmatrix}
\left[
  \begin{array}{c}
    p_T \\
    v_T \\
    \ell \\
  \end{array}
\right]
 + r(\hat{\theta}  \my{\mu}  - w \hat{\my{g}} ),
\end{aligned}
\end{align}
where $\hat{\theta} \my{p}_o $ is the new measurement that is pseudo-linear in the target's state variables.

\subsection{Bearing-angle estimation algorithm}

Combining  \eqref{eq_pseudo_linear_measurement_g_equation} and \eqref{eq_pseudo_linear_measurement_theta_equation} gives the compact form of the measurement equation:
\begin{align}\label{eq_pseudo_linear_measurement_equations}
\my{z} = \my{H} \my{x}  + \my{\nu} ,
\end{align}
where
\begin{subequations}
\begin{align}
\my{z} &=
	\begin{bmatrix}
	\my{P}_{\hat{g}} \my{p}_o   \\
	\hat{\theta} \my{p}_o
	\end{bmatrix}\in\mathbb{R}^6, \\
\my{H}& =
	\begin{bmatrix}
	\my{P}_{\hat{g}}  & \my{0}_{3\times 3} & \my{0}_{3\times 1} \\
	\hat{\theta} \my{I}_{3\times 3} & \my{0}_{3\times 3}  & -\hat{\my{g}}
	\end{bmatrix}\in\mathbb{R}^{6\times7},
	\label{eq_matrix_H}	\\
\my{\nu}  &=
	\begin{bmatrix}
	r \my{P}_{\hat{g}} \my{\mu}  \\
	r (\hat{\theta}  \my{\mu}  - w \hat{\my{g}} )
	\end{bmatrix}
	\in\mathbb{R}^6.
	\label{eq_final_measurement_noise}
\end{align}
\end{subequations}
Here, $\nu$ can be rewritten as a matrix form
\begin{align*}
    \nu=E
    \begin{bmatrix}
        \mu \\ w
    \end{bmatrix},
\end{align*}
where
\begin{align}\label{eq_E_mat}
    E=r
    \begin{bmatrix}
        P_{\hat{g}} & 0_{3\times 1}\\
        \hat{\theta}I_{3\times 3} & -\hat{g}
    \end{bmatrix}\in\mathbb{R}^{6\times 4}.
\end{align}
As a result, $\nu$ can be approximately treated as a linear transformation of Gaussian noises.
Its covariance matrix can be calculated as
\begin{align*}
\my{\Sigma}_{\my{\nu}}  = E
\begin{bmatrix}
\sigma_\mu^2 I_{3\times 3} & 0_{3\times1}\\
0_{1\times 3} & \sigma_w^2
\end{bmatrix}
E^\mathrm{T}\in\mathbb{R}^{6\times6}.
\end{align*}
Although the quantities in $E$ such as $\hat{g}$ and $\hat{\theta}$ contain measurement noises, it is a common practice to treat them as deterministic quantities. Otherwise, if, for example, $\hat{g}$ is split to $\hat{g}=g+\mu$ and we consider the noise separately, the expression of $\nu$ would be a complex function of the true values and the noises. Since the true values are unknown, the covariance cannot be calculated.
Moreover, $r $ in \eqref{eq_E_mat} is the true target range, which is unknown. We can use the estimated value $\hat{r} =\|\hat{\my{p}}_T -\my{p}_o \|$ to replace it in implementation. Here, $\hat{p}_T\in\mathbb{R}^3$ is the estimated value of the target's position. This technique has been used in bearing-only target estimation \citep{He2018, Li2022}.

With the state transition equation \eqref{eq_state_transition} and the measurement equation \eqref{eq_pseudo_linear_measurement_equations}, the bearing-angle estimator can be realized by the Kalman filter.
For a quick reference, we list the steps below.
The prediction steps are
\begin{align*}
\hat{\my{x}}^{-}(t_k) &= \my{F}\hat{\my{x}}(t_{k-1}), \\
\my{P}^{-}(t_k) &= \my{F}\my{P}(t_{k-1})\my{F}^\mathrm{T} + \my{\Sigma}_q,
\end{align*}
where $\hat{\my{x}}^{-}(t_k)\in\mathbb{R}^7$ and $\my{P}^{-}(t_k)\in\mathbb{R}^{7\times7}$ are the prior estimated state and covariance matrix, respectively.
The correction steps are
\begin{align*}
\my{K}(t_k) &= \my{P}^{-}(t_k)\my{H}^\mathrm{T}(t_k)\left[\my{H}(t_k)\my{P}^{-}(t_k)\my{H}^\mathrm{T}(t_k)+\my{\Sigma}_\nu\right]^{\dagger}, \\
\hat{\my{x}}(t_k) &= \hat{\my{x}}^{-}(t_k) + \my{K}(t_k)\left[\my{z}(t_k)-\my{H}(t_k)\hat{\my{x}}^{-}(t_k)\right], \\
\my{P}(t_k) &=\left[\my{I}_{7\times 7} -\my{K}(t_k)\my{H}(t_k) \right]\my{P}^{-}(t_k),
\end{align*}
where $\my{K}(t_k)\in\mathbb{R}^{7\times6}$ is the Kalman gain matrix, $\hat{\my{x}}(t_k) $  and $\my{P}(t_k)$ are posterior estimated state and covariance matrix, and symbol $\dagger$ denotes the pseudoinverse.
The usage of pseudoinverse in the Kalman filter is a common practice to prevent the situation that $\my{H}(t_k)\my{P}^{-}(t_k)\my{H}^\mathrm{T}(t_k)+\my{\Sigma}_\nu$ is rank deficient \citep{YOSHIKAWA1972,Kulikov2018}.

\section{Observability Analysis by Kalman's Criterion}\label{sec_analysis_of_observability_matrix}

Although an additional angle measurement is adopted in the bearing-angle estimator, it is nontrivial to see whether this additional measurement can improve the system's observability because an additional unknown variable, the target's physical size, is also required to estimate. It is therefore necessary to study the observability conditions under which the target's motion can be successfully estimated.

In this and the next sections, we present two methods to analyze the observability conditions. The first method, as presented in this section, relies on Kalman's observability criterion, which is to check the rank of the observability matrix of a linear system. The second method, as presented in the next section, relies on solving a set of linear equations.
Both methods have been adopted in the literature to analyze the observability of estimators \citep{Zhao2015, Fogel1988}.
For the bearing-angle estimator, the first method considers the specific dynamics of the filter but is not able to handle the case when the target's motion has a higher order.
The second method can handle the high-order motion of the target but does not consider the dynamics of the filter. We will show that the conclusions given by the two methods are consistent.
In both of the methods, we consider the case where $\ell$ is invariant.

\subsection{The observability matrix}

Consider a time horizon of $k\geq 3$ consecutive steps.
The observability matrix of the system of \eqref{eq_matrix_H} and \eqref{eq_matrix_A} can be calculated as
\begin{align}\label{eq_Qo}
	\my{Q}=
	\begin{bmatrix}
	\my{H}(t_1) \\
	\my{H}(t_2)\my{F} \\
	\my{H}(t_3)\my{F}^2 \\
	\cdots \\
	\my{H}(t_k)\my{F}^{k-1} \\
	\end{bmatrix}\in\mathbb{R}^{6k\times7}.
\end{align}
Substituting the expressions of $F$ and $H$ in \eqref{eq_matrix_A} and \eqref{eq_matrix_H} into \eqref{eq_Qo} yields
\begin{align*}
\my{Q}=
\left[
\begin{array}{ccc}
\my{P}_g(t_1) & \my{0}_{3\times 3} & \my{0}_{3\times 1} \\
\theta(t_1)\my{I}_{3\times 3} & \my{0}_{3\times 3}  & -\my{g}(t_1) \\
\hdashline
\my{P}_g(t_2) & \delta t\my{P}_{g}(t_2) & \my{0}_{3\times 1} \\
\theta(t_2)\my{I}_{3\times 3} & \delta t\theta(t_2)\my{I}_{3\times 3}  & -\my{g}(t_2) \\
\hdashline
\vdots & \vdots & \vdots \\
\hdashline
\my{P}_g(t_k) & (k-1)\delta t\my{P}_{g}(t_k) & \my{0}_{3\times 1} \\
\theta(t_k)\my{I}_{3\times 3} & (k-1)\delta t\theta(t_k)\my{I}_{3\times 3}  & -\my{g}(t_k)\\
\end{array}
\right].
\end{align*}
Note that the noises in the bearing and angle measurements are neglected when we analyze the fundamental observability property.
After a series of elementary row transformations in $\my{Q}$, we can obtain
\begin{align}\label{eq_Qo_2}
\my{Q}
\rightarrow
\begin{bmatrix}
\my{I}_{3\times 3} & \my{0}_{3\times 3} & -\my{g}(t_1)/\theta(t_1) \\
\my{0}_{3\times 3} & \my{I}_{3\times 3} & -\delta\my{v}(t_2)/\ell \\
\vdots & \vdots & \vdots \\
\my{0}_{3\times 3} & \my{I}_{3\times 3} & -\delta\my{v}(t_k)/\ell \\
\hdashline
\my{0}_{3k\times 3} & \my{0}_{3k\times 3} & \my{0}_{3k\times 1}
\end{bmatrix},
\end{align}
where
\begin{align*}
\delta\my{v}(t_k) \doteq \my{v}_T(t_k) - \my{v}_o(t_k)
\end{align*}
is the relative velocity.

In the following two subsections, we analyze the rank of the observability matrix in two scenarios where the observer moves with zero and nonzero acceleration, respectively. In the two scenarios, the target is always assumed to move with a constant velocity:
\begin{align*}
\my{v}_T(t_k) = \my{v}_T^\text{const}.
\end{align*}

\subsection{Case 1: the observer's velocity is constant}
Denoted $\my{v}_o\in \mathbb{R}^3$ as the velocity of the observer.
Consider the case where the observer has a constant velocity $\my{v}_o^\text{case1}(t_i)=\my{v}_o^\text{const}$ for any $i\in\{1,\dots,k\}$.
Then, the relative velocity is also constant:
\begin{align}\label{eq_delta_vel_case1}
\delta \my{v}^\text{case1}(t_i) = \my{v}_T^\text{const} - \my{v}_o^\text{const} = \delta\my{v}^\text{const}.
\end{align}
Substituting \eqref{eq_delta_vel_case1} into \eqref{eq_Qo_2} and conducting elementary row transformation yields
\begin{align}\label{eq_Qo_3}
\my{Q}^\text{case1}
\rightarrow
\left[
\begin{array}{cc:c}
\my{I}_{3\times 3} & \my{0}_{3\times 3} & -\my{g}(t_1)/\theta(t_1) \\
\my{0}_{3\times 3} & \my{I}_{3\times 3} & -\delta\my{v}^\text{const}/\ell \\
\hdashline
\my{0}_{6(k-1)\times 3} & \my{0}_{6(k-1)\times 3} & \my{0}_{6(k-1)\times 1}
\end{array}\right].
\end{align}
Since the upper $6\times7$ block of \eqref{eq_Qo_3} has full row rank and the lower block is zero, the rank of $\my{Q}^\text{case1}$ is
\begin{align*}
\text{rank}\left(\my{Q}^\text{case1}\right) = 6.
\end{align*}
Since the number of states is seven and the rank is six, we know there is \emph{one unobservable mode}.
To identify this unobservable mode, we calculate the unobservable subspace, which is the null space of $\my{Q}$:
\begin{align}\label{eq_unobservable_subspace}
\text{Null}\left(\my{Q}^\text{case1}\right) = \text{span}\left\{
\begin{bmatrix}
\my{g}(t_1)/\theta(t_1)  \\
\delta\my{v}^\text{const}/\ell \\
1
\end{bmatrix}\right\}.
\end{align}
According to \eqref{eq_unobservable_subspace}, the unobservable mode is
\begin{align}
x^T
\left[
\begin{array}{c}
\my{g}(t_1)/\theta(t_1)  \\
\delta\my{v}^\text{const}/\ell \\
1
\end{array}
\right]
=
\my{p}_T^\mathrm{T}\dfrac{\my{g}(t_1)}{\theta(t_1)}+
\my{v}_T^\mathrm{T}\dfrac{\delta\my{v}^\text{const}}{\ell} + \ell.
\label{eq_unobservable_mode}
\end{align}%
Although there is only one unobservable mode, this mode given in \eqref{eq_unobservable_mode} involves all the states including the target's position, velocity, and physical size. It suggests that the estimation of the three quantities is coupled. In conclusion, we know that, if the target moves with a constant velocity, its states are unobservable when the observer moves with a constant velocity.

\subsection{Case 2: the observer's velocity is time-varying}

We now consider the case where the observer has nonzero acceleration so that its velocity is time-varying across the time horizon from $t_1$ to $t_k$.

Denote $\my{a}_o(t_i)\in\mathbb{R}$ as the observer's acceleration, which can be approximated as
\begin{align}\label{eq_acc}
\my{a}_o(t_i) &\approx
\dfrac{\my{v}_o(t_i) - \my{v}_o(t_{i-1})}{\delta t} \nonumber\\
&=-\dfrac{\left[\my{v}_T^\text{const} - \my{v}_o(t_i)\right] - \left[\my{v}_T^\text{const} - \my{v}_o(t_{i-1})\right]}{\delta t} \nonumber\\
&=-\dfrac{\delta \my{v}(t_i) - \delta \my{v}(t_{i-1})}{\delta t}.
\end{align}
Substituting \eqref{eq_acc} into \eqref{eq_Qo_2} and performing elementary row transformation yields
\begin{align}\label{eq_Q_case2_final}
\my{Q}^\text{case2}
\rightarrow
\left[\begin{array}{ccc}
\my{I}_{3\times 3} & \my{0}_{3\times 3} & -\my{g}(t_1)/\theta(t_1) \\
\my{0}_{3\times 3} & \my{I}_{3\times 3} & -\delta\my{v}(t_2)/\ell \\
\my{0}_{3\times 3} & \my{0}_{3\times 3} & \delta t \my{a}_o(t_3)/\ell \\
\hdashline
\vdots & \vdots &\vdots \\
\my{0}_{3\times 3} & \my{0}_{3\times 3} & \delta t \my{a}_o(t_k)/\ell \\
\my{0}_{3k\times 3} & \my{0}_{3k\times 3} & \my{0}_{3k\times 1}
\end{array}\right].
\end{align}
The upper $6\times7$ block in \eqref{eq_Q_case2_final} has full column rank.
Therefore, if $a_o(t_i)\ne0$ for any $i\geq3$, then
\begin{align*}
\text{rank}\left(\my{Q}^\text{case2}\right) = 7,
\end{align*}
Which is the same as the number of estimated states.
Therefore, the target's state is observable when the observer moves with nonzero acceleration.

\subsection{Summary of this section}

From the above analysis, we know that when the target has a constant velocity, its states including its position, velocity, and physical size are observable if and only if the observer has non-zero accelerations.

The critical difference of this condition from the bearing-only case is that the target's states are still observable \emph{even if the observer moves along the bearing vector} towards or backward the target.
By contrast, for a bearing-only estimator, moving along the bearing vector is insufficient to recover the target's motion. Therefore, the additional lateral motion of the observer required in the bearing-only case is \emph{not} required in the bearing-angle case anymore, which provides better flexibility for designing the observer's motion.

\section{Observability Analysis by Solving Linear Equations}\label{sec_observability_criteria}

This section extends the observability condition obtained in the last section to more general cases where the target's velocity does not have to be constant.

\subsection{Problem formulation}

The observability problem that we aim to solve is to determine whether $\my{p}_T(t)$ can be recovered from $\my{p}_o(t)$ and $g(t),\theta(t)$.

Suppose the target's motion can be described by an $n$th-order polynomial during a time interval:
\begin{align}\label{eq_target_nth_Order}
	\my{p}_T(t)=\my{b}_0+\my{b}_1t+\cdots+\my{b}_nt^n,
\end{align}
where $\my{b}_0, \my{b}_1, \cdots, \my{b}_n\in\mathbb{R}^3$ are unknown constant vectors.
If we can determine the values of $\{b_i\}_{i=0}^n$, then we can determine the target's motion and hence it is observable.
Although polynomials cannot represent all trajectories, they can effectively approximate a majority of them according to the method of Taylor expansion. This is especially true if we consider a short time horizon. This kind of technique has been adopted in the observability analysis of bearing-only target motion estimation tasks~\citep{Nardone1981, Lee2010}.

Suppose the observer's motion is described by
\begin{align*}
	\my{p}_o(t)=\my{c}_0+\my{c}_1t+\cdots+\my{c}_nt^n+\my{h}(t),
\end{align*}
where $\my{c}_0, \my{c}_1, \cdots, \my{c}_n\in\mathbb{R}^3$ are constant parameters, and
\begin{align}\label{eq_definition_h}
\my{h}(t) = \my{d}_1 t^{n+1}+\my{d}_2t^{n+2}+\cdots
\end{align}
represents \emph{higher-order} motion with $\my{d}_1, \my{d}_2, \cdots\in\mathbb{R}^3$.
It can be verified that the derivatives of $\my{h}(t)$ satisfy $\my{h}^{(i)}(0)=\my{0}_{3\times 1}$ for $i=0,1,\cdots, n$.
Let $\my{s}(t)\in\mathbb{R}^3$ be the relative motion between the target and the observer:
\begin{align}\label{eq_relative_motion}
	\my{s}(t)&\doteq\my{p}_T(t)-\my{p}_o(t)  \nonumber\\
	&\doteq\my{s}_0+\my{s}_1t+\cdots+\my{s}_nt^n+\my{h}(t),
\end{align}
where $\my{s}_i = \my{d}_i - \my{c}_i\in\R^3$ for $i = 0,1,\cdots, n$.

If we can determine $\{s_i\}_{i=0}^n$, then $s(t)$ and hence $p_T(t)$ can be determined.
Therefore, we next study under what conditions $\{s_i\}_{i=0}^n$ can be uniquely determined.
Since $\my{p}_T(t)-\my{p}_o(t)=g(t)r(t)$ according to \eqref{eq_bearing_measure} and $r(t)=\ell/\theta(t)$ according to \eqref{eq_theta_measure}, we have
$$s(t)=\my{p}_T(t)-\my{p}_o(t)=g(t)r(t)=\frac{g(t)}{\theta(t)}\ell. $$
Substituting the above equation into \eqref{eq_relative_motion} yields
\begin{align}\label{eq_st_tem}
\my{s}_0+\my{s}_1t+\cdots+\my{s}_nt^n+\my{h}(t)=\frac{g(t)}{\theta(t)}\ell.
\end{align}
Here, $\my{s}_0, \cdots, \my{s}_n, \ell$ are unknowns to be determined and $\my{g}(t),\theta(t),\my{h}(t)$ are known.
Equation~\eqref{eq_st_tem} can be reorganized to a linear equation:
\begin{align}\label{eq_linear_equations}
	\my{A}(t)\my{X} = \my{h}(t),
\end{align}
where
\begin{align*}
\my{X}&=
\begin{bmatrix}
\my{s}_0^\mathrm{T}, \my{s}_1^\mathrm{T}, \cdots, \my{s}_n^\mathrm{T}, \ell
\end{bmatrix}^\mathrm{T}\in\mathbb{R}^{3n+4},
\end{align*}
and
\begin{align}\label{eq_original_A}
\my{A}(t)&=
\begin{bmatrix}
\my{I}_{3\times3}, t\my{I}_{3\times3}, \cdots, t^n\my{I}_{3\times3}, \rho(t)
\end{bmatrix}\in\mathbb{R}^{3\times(3n+4)},
\end{align}
where
\begin{align}\label{eq_rho_denote}
\rho(t)&\doteq-\dfrac{\my{g}(t)}{\theta(t)}\in\R^3.
\end{align}
Therefore, the problem that we aim to solve becomes determining whether $X$ can be uniquely solved from \eqref{eq_linear_equations}.

\subsection{Necessary and sufficient observability condition}

We next present a necessary and sufficient condition under which the solution $X$ of \eqref{eq_linear_equations} is unique.

\begin{theorem}[(Necessary and sufficient observability condition)]
\label{theorem_observability_confition}
The target's motion $p_T(t)$ can be uniquely determined by the observer's motion $p_o(t)$, the bearing $g(t)$, and the angle $\theta(t)$ if and only if
\begin{align*}
\my{h}(t)\neq\my{0}_{3\times1},
\end{align*}
which means that the order of the observer's motion must be greater than the target.
\end{theorem}
\begin{proof}
Since the row number of $\my{A}(t)$ is less than its column number, \eqref{eq_linear_equations} is an under-determined system whose solution cannot be uniquely determined.
However, in the continuous time domain, we can use additional higher derivatives of this equation to uniquely determine $X$.

In particular, taking the $i$th-order derivative on both sides of \eqref{eq_linear_equations} gives $A^{(i)}(t)X=h^{(i)}(t)$. Consider any integer $N$ satisfying $N\ge n+1$. Combining the equations with $i\in\{0,1,\dots,N\}$  gives
\begin{align}\label{eq_new_linear_equtions}
	\bar{\my{A}}(t)\my{X} = \bar{\my{h}}(t),
\end{align}
where
\begin{align}\label{eq_new_A}
	\bar{\my{A}}(t) =
\left[
  \begin{array}{c}
    \my{A}(t) \\
	\my{A}^{'}(t) \\
	\vdots \\
	\my{A}^{(N)}(t)
  \end{array}
\right],\qquad
\bar{\my{h}}(t)\left[
  \begin{array}{c}
    \my{h}(t)\\
	\my{h}^{'}(t)\\
	\vdots\\
	\my{h}^{(N)}(t)\\
  \end{array}
\right].
\end{align}
Here, $\bar{\my{A}}(t)\in\mathbb{R}^{(3N+3)\times (3n+4)}$ and $\bar{\my{h}}(t)\in\mathbb{R}^{3N+3}$.
Since $N\ge n+1$, $\bar{A}(t)$ is a tall matrix and \eqref{eq_new_linear_equtions} is an over-determined system.

We next examine when $\bar{A}(t)$ has full column rank.
Substituting \eqref{eq_original_A} into $\bar{A}(t)$ yields
\begin{align*}
\bar{\my{A}}(t)=
	\left[\begin{array}{cccc:c}
	\my{I}_{3\times3}& t\my{I}_{3\times3}& \cdots& t^n\my{I}_{3\times3}& \rho(t) \\
	\my{0}_{3\times3}& \my{I}_{3\times3}& \cdots& nt^{n-1}\my{I}_{3\times3}& \rho^{'}(t) \\
	\vdots & \vdots & \ddots & \vdots & \vdots \\
	\my{0}_{3\times3}& \my{0}_{3\times3}& \cdots& n!\my{I}_{3\times3}& \rho^{(n)}(t) \\
	\hdashline
	\my{0}_{3\times3}& \my{0}_{3\times3}& \cdots& \my{0}_{3\times3}& \rho^{(n+1)}(t) \\
\vdots & \vdots & \vdots & \vdots & \vdots  \\
\my{0}_{3\times3}& \my{0}_{3\times3}& \cdots& \my{0}_{3\times3}& \rho^{(N)}(t) \\
	\end{array}\right].
\end{align*}
Since the top-left block of $\bar{A}(t)$ is a full-rank square matrix, $\bar{\my{A}}(t)$ has full column rank if and only if there exists $i\in\{n+1,\dots,N\}$ such that
\begin{align}\label{eq_observability_criteria_2}
	\rho^{(i)}(t)\neq \my{0}_{3\times1}.
\end{align}
Since $\rho(t)=-g(t)/\theta(t)$ as shown in \eqref{eq_rho_denote} and $g(t)/\theta(t)=(\my{s}_0+\my{s}_1t+\cdots+\my{s}_nt^n+\my{h}(t))/\ell$ as shown in \eqref{eq_st_tem}, we can rewrite \eqref{eq_observability_criteria_2} to
\begin{align}\label{eq_critia_2}
-\dfrac{1}{\ell}(\my{s}_0+\my{s}_1t+\cdots+\my{s}_nt^n+\my{h}(t))^{(i)}\neq \my{0}_{3\times1}.
\end{align}
Since $i\ge n+1$, \eqref{eq_critia_2} is equivalent to
\begin{align}\label{eq_observability_criteria_final}
\my{h}^{(i)} (t) \neq \my{0}_{3\times1}.
\end{align}
According to the definition of $\my{h}(t)$ in \eqref{eq_definition_h}, the condition in \eqref{eq_observability_criteria_final} is equivalent to
\begin{align*}
\my{h}(t)\neq \my{0}_{3\times1}.
\end{align*}
The proof is complete.
\end{proof}

Some important remarks about Theorem~\ref{theorem_observability_confition} are given below.

1) The necessary and sufficient condition suggested by Theorem~\ref{theorem_observability_confition} is that the observer should have higher-order motion than the target.
For example, when the target is stationary, the observer should move with a nonzero velocity. When the target moves with a constant velocity, the observer should move with a nonzero acceleration.

2) The necessary and sufficient condition given by Theorem~\ref{theorem_observability_confition} has a \emph{key difference} from the bearing-only case that the higher-order motion in the bearing-angle case is \emph{not} required to be orthogonal to the bearing vector, making the bearing-angle approach more flexible than the bearing-only one.
For example, the bearing-angle approach can estimate the target's motion even if the observer simply moves along the bearing vector.

3) In the special case where the target moves with a constant velocity, the condition in Theorem~\ref{theorem_observability_confition} is consistent with the one obtained in Section~\ref{sec_analysis_of_observability_matrix}. Although the condition in Theorem~\ref{theorem_observability_confition} allows more general target motion, the analysis in Section~\ref{sec_analysis_of_observability_matrix} is still meaningful since it is directly related to the dynamic model used in the pseudo-linear Kalman filter.

4) In practice, we would not estimate the target's motion by using the method of solving an equation like \eqref{eq_new_linear_equtions}. That is because such a method involves calculating high-order derivatives, which are challenging to obtain accurately in practice. The role of this equation is to provide a fundamental perspective on whether there is sufficient information to uniquely recover the target's motion.

\subsection{Number of observations required}

\begin{figure*}[!ht]
\normalsize
\begin{align}
\label{eq_A_22}
\tilde{A}
\rightarrow
\left[\begin{array}{ccccc:c}
\my{I} & t_1\my{I} & \cdots & t_1^{n-1}\my{I} & t_1^n\my{I} & \rho(t_1) \\
\my{0} & \my{I} & \cdots & {\Delta(t_2^{n-1}, t_1^{n-1})}\my{I} & {\Delta(t_2^n, t_1^n)}{}\my{I} & {\Delta(\rho(t_2),\rho(t_1) )}{} \\
\vdots & \vdots & \ddots & \vdots & \vdots & \vdots \\
\my{0} & \my{0} & \cdots & (n-1)!\my{I} & {\Delta^{n-1}(t_n^n, \cdots , t_1^n)}{} & {\Delta^{n-1}(\rho(t_n),\cdots,\rho(t_1) )}{} \\
\my{0} & \my{0} & \cdots & \my{0} & n!\my{I} &  {\Delta^{n}(\rho(t_{n+1}),\cdots,\rho(t_1) )}{} \\
\hdashline
\my{0} & \my{0} & \cdots & \my{0} & \my{0} & {\Delta^{n+1}(\rho(t_{n+2}),\cdots,\rho(t_1) )}{} \\
\vdots & \vdots & \vdots & \vdots & \vdots & \vdots  \\
\my{0} & \my{0} & \cdots & \my{0} & \my{0} & {\Delta^{N-1}(\rho(t_N),\cdots,\rho(t_1) )}{} \\
\end{array}
\right]
\end{align}
\hrulefill
\vspace*{4pt}
\end{figure*}

It is of practical importance to study how many discrete observations are required to recover the target's motion. Although Theorem~\ref{theorem_observability_confition} gives an observability condition, it does not answer this question because it is based on the continuous time domain. We next answer this question by exploring multiple discrete time steps.

\begin{theorem}[(Number of discrete observations)]\label{theorem_observation_number}
If the observer's motion satisfies the observability condition in Theorem~\ref{theorem_observability_confition}, it is necessary and sufficient to use at least $n+2$ observations to recover the target's motion. Here, $n$ is the order of the target's polynomial motion as shown in \eqref{eq_target_nth_Order}.
\end{theorem}
\begin{proof}
Consider $t_1,\dots,t_N$ time instances. Each time instance corresponds to an equation like \eqref{eq_linear_equations}: $\my{A}(t_i)\my{X} = \my{h}(t_i)$ for $i=1,\dots,N$.
Combining these equations gives
\begin{align}\label{eq_convergence_linear_eqs}
\tilde{\my{A}}\my{X}=\tilde{\my{h}},
\end{align}
where
\begin{align}\label{eq_new_A_2}
	\tilde{\my{A}} =
\left[
  \begin{array}{c}
    \my{A}(t_1) \\
	\vdots \\
	\my{A}(t_N)
  \end{array}
\right],\qquad
\tilde{\my{h}}\left[
  \begin{array}{c}
    \my{h}(t_1)\\
	\vdots\\
	\my{h}(t_N)\\
  \end{array}
\right].
\end{align}
Here, $\tilde{\my{A}}\in\mathbb{R}^{(3N) \times (3n+4)}$ and $\tilde{\my{h}}\in\mathbb{R}^{3N}$.

(\emph{Necessity}) Since $X\in\R^{3n+4}$, we need at least $N\ge n+2$ observations so that $\tilde{\my{A}}$ is a tall matrix and hence \eqref{eq_convergence_linear_eqs} is an over-determined system.

(\emph{Sufficiency})
Suppose we have $N\ge n+2$ discrete observations.
Substituting \eqref{eq_original_A} into \eqref{eq_new_A_2} yields
\begin{align*}
\tilde{A} =
\begin{bmatrix}
\my{I}_{3\times 3} & t_1\my{I}_{3\times 3} & \cdots & t_1^n\my{I}_{3\times 3} & \rho(t_1) \\
\my{I}_{3\times 3} & t_2\my{I}_{3\times 3} & \cdots & t_2^n\my{I}_{3\times 3} & \rho(t_2)  \\
\vdots & \vdots && \vdots & \vdots  \\
\my{I}_{3\times 3} & t_{n+1}\my{I}_{3\times 3} & \cdots & t_{n+1}^n\my{I}_{3\times 3} & \rho(t_{n+1}) \\
\my{I}_{3\times 3} & t_{n+2}\my{I}_{3\times 3} & \cdots & t_{n+2}^n\my{I}_{3\times 3} & \rho(t_{n+2})\\
\vdots & \vdots & \vdots & \vdots & \vdots\\
\my{I}_{3\times 3} & t_{N}\my{I}_{3\times 3} & \cdots & t_{N}^n\my{I}_{3\times 3} & \rho(t_{N})\\
\end{bmatrix}.
\end{align*}
Starting from the last line in $\tilde{A}$, subtract the previous line from each subsequent line, and repeat this process.
Finally, we can obtain \eqref{eq_A_22} (the equation is too long and located at the top of another page).
Here, $\Delta^n$ represents the $n$th-order time difference \citep{MilneThomson2000}.
For example, $\Delta (a_2, a_1)=(a_2-a_1)/\delta t$, $\Delta^2 (a_3, a_2, a_1) = \Delta (\Delta(a_3, a_2), \Delta(a_2, a_1))=[(a_3-a_2)/\delta t-(a_2-a_1)/\delta t]/\delta t$.
When $\delta t$ is sufficiently small, the time difference is an approximation of the derivative.
When the observability condition in Theorem~\ref{theorem_observability_confition} is satisfied, there exists $i\ge n+1$ such that $\rho^{(i)}(t)\neq 0$ as shown in \eqref{eq_observability_criteria_2}. As a result, there exists $i\ge n+1$ such that
\begin{align*}
\Delta^{i}(\rho(t_{i+1}),\cdots,\rho(t_1))\neq \my{0}.
\end{align*}
The above implication is valid because $\Delta^{i}$ is an approximation of the $i$th-order derivative when $\delta t$ is sufficiently small.
Then, $\tilde{A}$ in \eqref{eq_A_22} has full column rank and hence \eqref{eq_convergence_linear_eqs} has a unique solution.
\end{proof}

Theorem~\ref{theorem_observation_number} suggests that when the target is stationary and hence $n=0$, at least two discrete observations are sufficient to localize the target. This is true even if the two observations are acquired when the observer moves along the bearing vector.
When the target moves with a constant velocity and hence $n=1$, at least three discrete observations are sufficient to localize the target, which is consistent with the results in Section~\ref{sec_analysis_of_observability_matrix}.

\section{Numerical Simulation Results}\label{sec_matlab_simulation}
\begin{figure*}[!t]
\centering
\subfloat[Scenario 1: Circular motion around the target. Both the bearing-only and bearing-angle approaches work well, but the bearing-angle one converges faster.]{
\includegraphics[width=1 \linewidth]{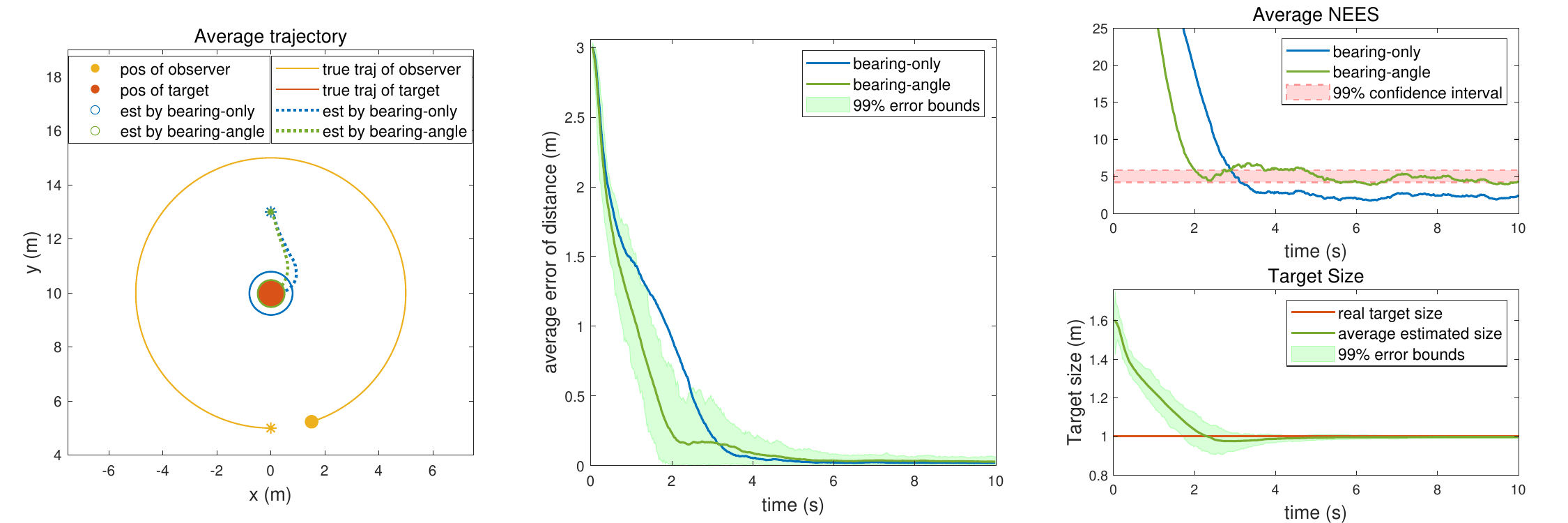}
\label{fig_matlab_1}
}
\hfill
\subfloat[Scenario 2: Straight motion towards and backwards the target. The bearing-only approach fails, but the bearing-angle approach works effectively.]{
\includegraphics[width=1 \linewidth]{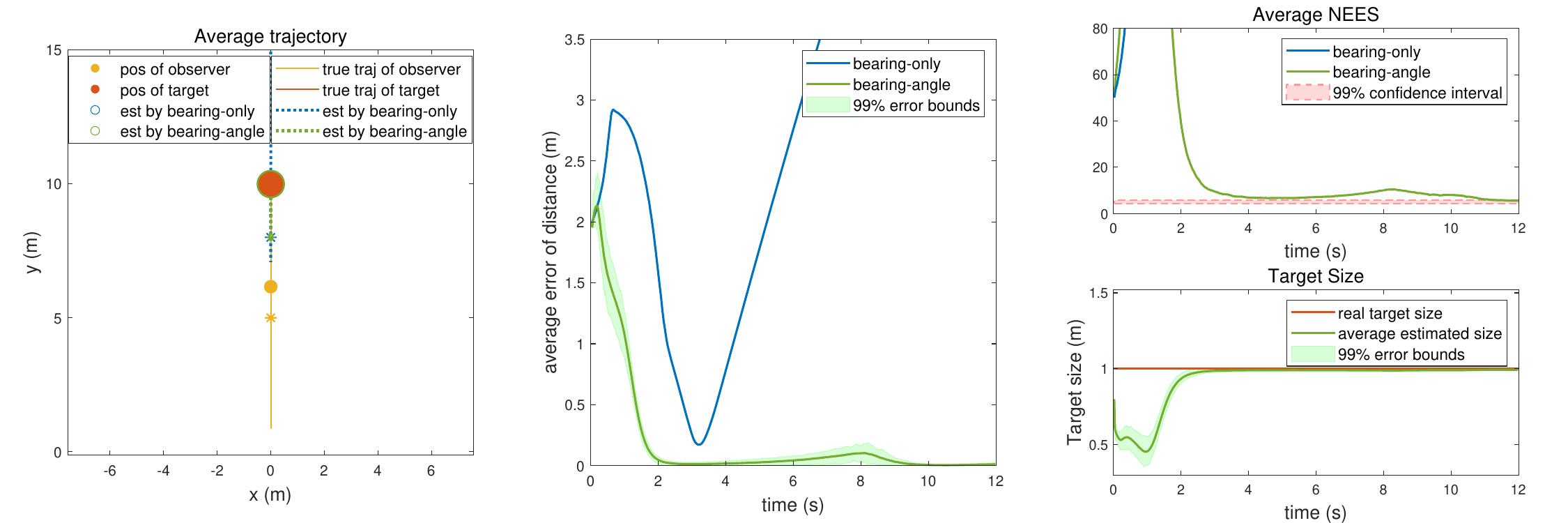}
\label{fig_matlab_2}
}
\hfill
\subfloat[
Scenario 3: Approaching the target by a guidance law. The bearing-only approach works unstably, but the bearing-angle approach works effectively.]{
\includegraphics[width=1 \linewidth]{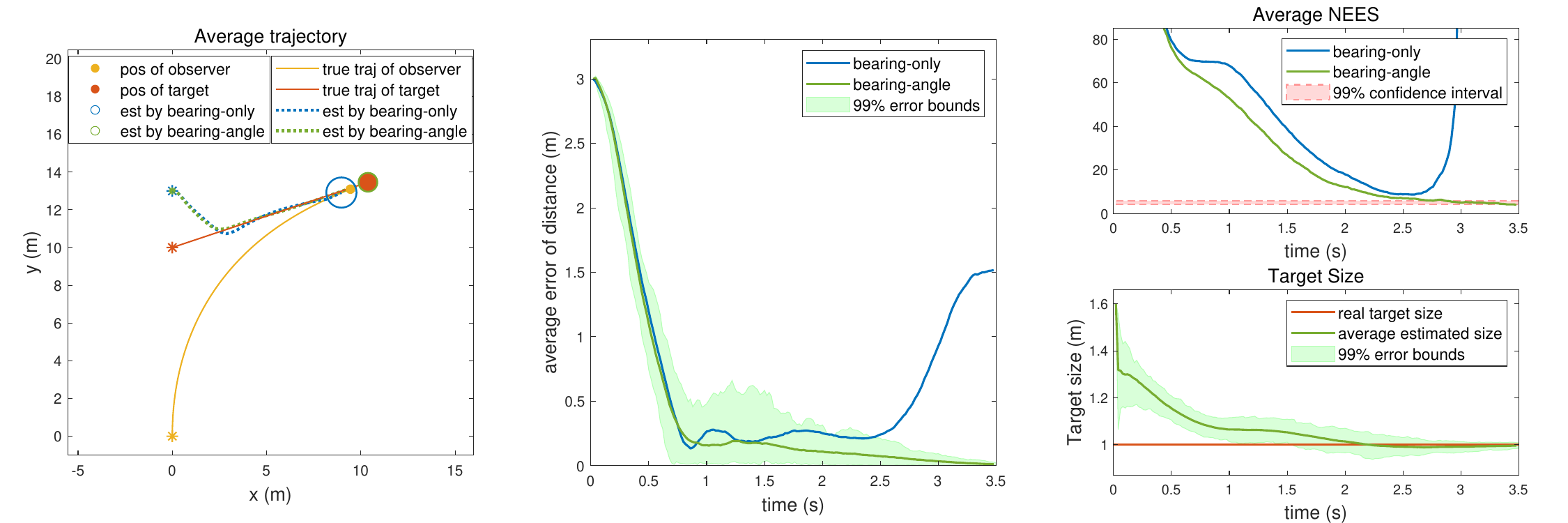}
\label{fig_matlab_3}
}
\caption{Numerical simulation results based on 100 Monte Carlo runs in three scenarios.}
\end{figure*}

\begin{figure*}[!t]
\label{fig_matlab_varying_ell}
\centering
\subfloat[
The observer moves around the square-shaped target. The target spins rapidly at $2\pi$~rad/s.]{
\includegraphics[width=1 \linewidth]{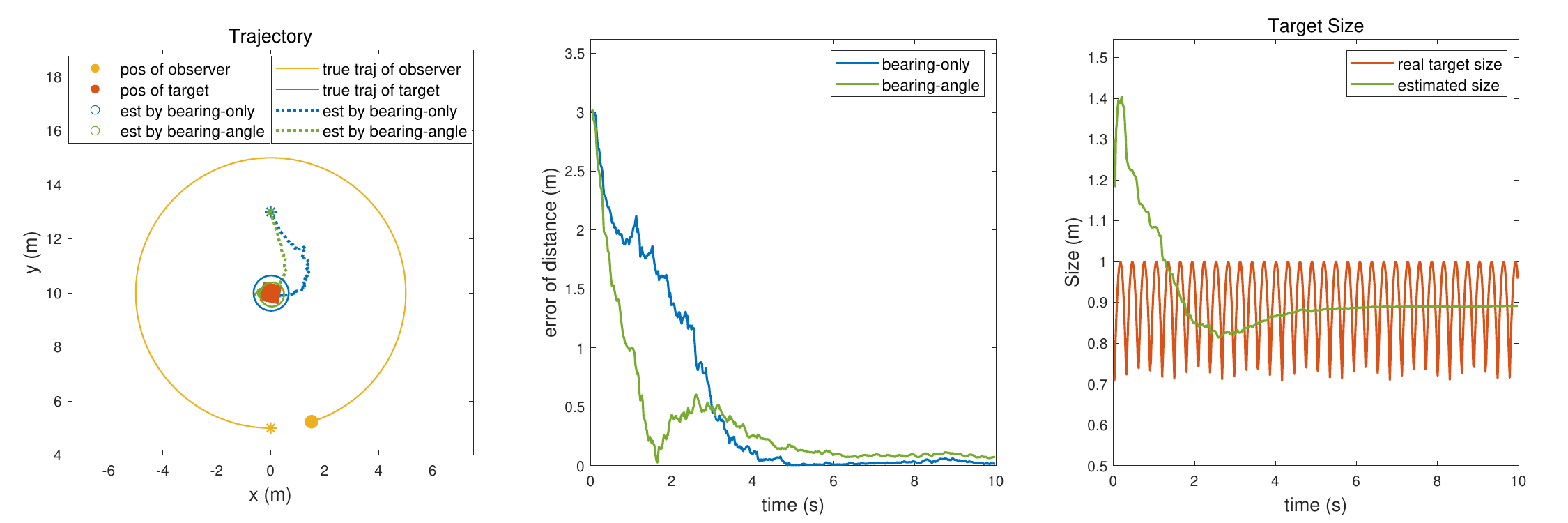}
\label{fig_matlab_4}
}
\hfill
\subfloat[
The observer moves along the bearing vector. The target's spinning speed is $\pi/8$~rad/s.]{
\includegraphics[width=1 \linewidth]{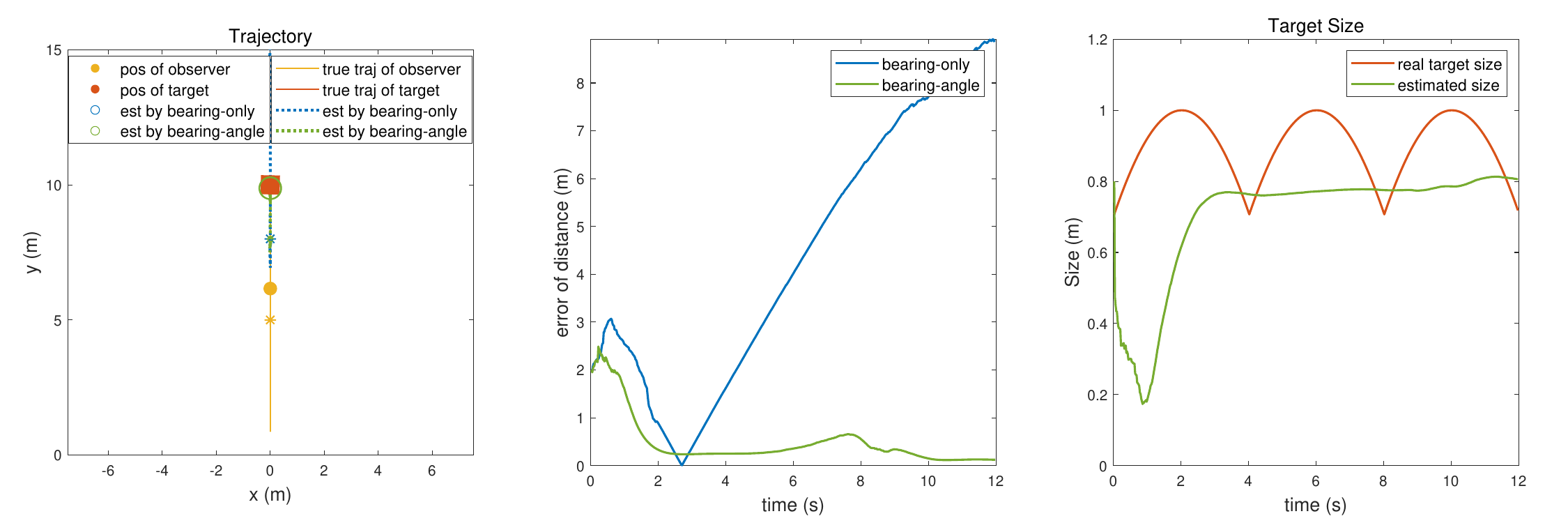}
\label{fig_matlab_pi_8}
}
\caption{Numerical simulation results for time-varying $\ell$.}
\end{figure*}
This section presents a set of numerical simulation results to demonstrate the effectiveness of the proposed bearing-angle approach.

The values of the parameters in two estimators are selected as $\sigma_v=10^{-3}$, $\sigma_l=10^{-4}$, $\sigma_\mu=0.01$, and $\sigma_w=0.01$.
The selection of these values is inspired by the measurement noises obtained in the AirSim simulation and real-world experiments as shown later.
The initial covariance matrix of the estimated states is set to $P(t_0)=0.1I$.
The target is a circle whose diameter is $\ell=1$.
The update rate of the system is $50$~Hz.
In addition, we use the same parameter values across all the simulation examples to verify the robustness of the algorithm.
Better performances can be achieved if the parameters are well-tuned for specific scenarios.
We perform $N_x=100$ Monte Carlo simulations for each scenario.

We use the normalized-estimation error squared (NEES) \citep{bar1998estimation} to analyze the consistency of the estimation algorithms.
In particular, the value of the average NEES is

\begin{align}
    \bar{\epsilon}_{\text{NEES}}=\dfrac{1}{N_x}\sum_{i=1}^{N_x}(x-\hat{x}_i)^\mathrm{T}P_i^{-1}(x-\hat{x}_i),
    \label{eq_nees}
\end{align}
where $\hat{x}_i$ is the estimated states in the $i$th simulation, and $P_i$ is the covariance matrix obtained from the estimator in the $i$th simulation.

Finally, image acquisition and visual detection are not considered in these numerical simulation scenarios. They will be considered in Section~\ref{sec_airsim_simulation} and Section~\ref{sec_real_world_experimental_validation}.

\subsection{Scenario 1: Circular motion around the target}
In the first scenario, the target is stationary and located at $\my{p}_T=[0, 10]^\mathrm{T}$.
The observer moves on a circle centered at the target with the speed of $3$~m/s (see Fig.~\ref{fig_matlab_1}).
The radius of the circle is $5$~m.
The initial estimates are $\hat{p}_o(t_0) = [0, 13]^\mathrm{T}$, $\hat{v_o}(t_0)=[0, 0]^\mathrm{T}$, $\hat{\ell}(t_0)=1.6$.
During this process, the bearing vector varies while the angle subtended by the target remains constant.
The angle measurement varies slightly due to the measurement noise.
This scenario is favorable to the conventional bearing-only approach because its observability condition that the target should be viewed from different angles is well satisfied \citep{Li2022}.

Fig.~\ref{fig_matlab_1} shows the estimation results by the two approaches of bearing-only and bearing-angle.
As can be seen, both algorithms perform well.
The convergence of the bearing-angle approach is faster than the bearing-only one, as shown in the middle and right subfigures of Fig.~\ref{fig_matlab_1}, due to the additional angle measurement.
The bearing-angle approach can successfully estimate the size of the target as shown in the right subfigure of Fig.~\ref{fig_matlab_1}.

\subsection{Scenario 2: Straight motion towards and backwards the target repeatedly}
In the second scenario, the target is also stationary but the observer moves along a straight line towards and backwards the target repeatedly (Fig.~\ref{fig_matlab_2}).
During this process, the bearing vector remains constant while the angle varies.
This scenario is most challenging for the bearing-only approach because its observability condition is not fulfilled.

In this simulation scenario, the target is stationary and located at $\my{p}_T(t_0)=[0, 10]^\mathrm{T}$.
The observer moves along a straight line towards and backwards the target with a constant acceleration of $-2$~$\text{m/s}^2$. The initial conditions are $v_o(t_0)=[0, 4]^\mathrm{T}$ and $\my{p}_o (t_0)= [0,5]^\mathrm{T}$.
The initial estimates are $\hat{p}_o(t_0) = [0, 8]^\mathrm{T}$, $\hat{v_o}(t_0)=[0, 0]^\mathrm{T}$, $\hat{\ell}(t_0)=0.8$.
In this scenario, the true bearing of the target relative to the observer remains unchanged though the bearing measurement may vary slightly due to the measurement noise.

Fig.~\ref{fig_matlab_2} shows the estimation results of the bearing-only and bearing-angle approaches.
As can be seen, the bearing-only approach diverges since its observability condition is not satisfied.
By contrast, the proposed bearing-angle approach converges, and is able to localize the target and estimate its size, which demonstrates the strong observability of the bearing-angle approach.
One may notice that the estimated size and the NEES value get worse first before converging.
This is because the noise level of the angle is set to be constant. Since the angle is small in the beginning, the noise-angle ratio is large, causing a relatively large NEES value.

\begin{figure*}[!t]
\centering
\subfloat[An AirSim simulation experimental scenario.]{
\includegraphics[width=0.5 \linewidth]{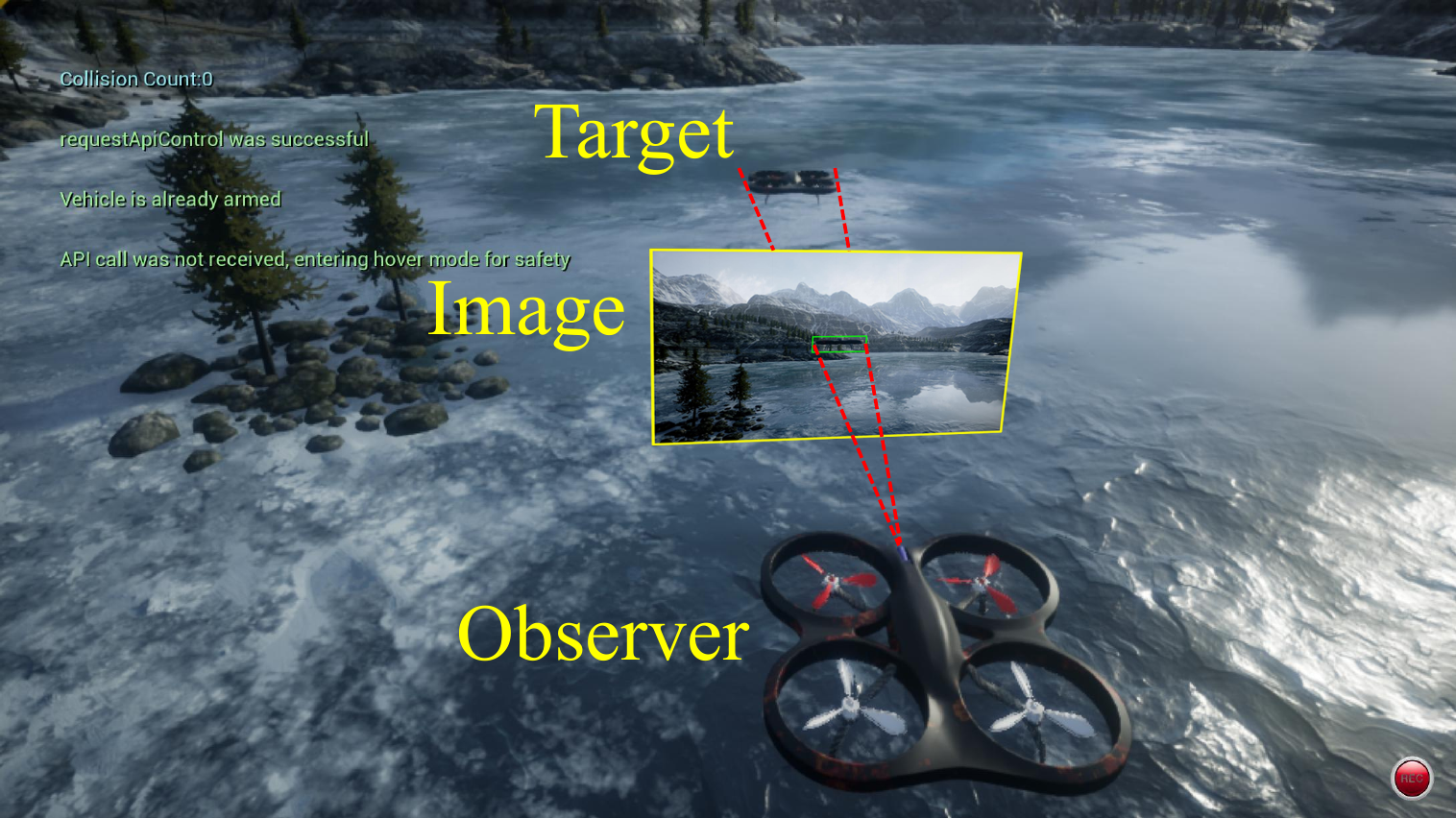}
\label{fig_architecture_airsim}
}
\subfloat[Samples of the dataset collected automatically in AirSim.]{
\includegraphics[width=0.5 \linewidth]{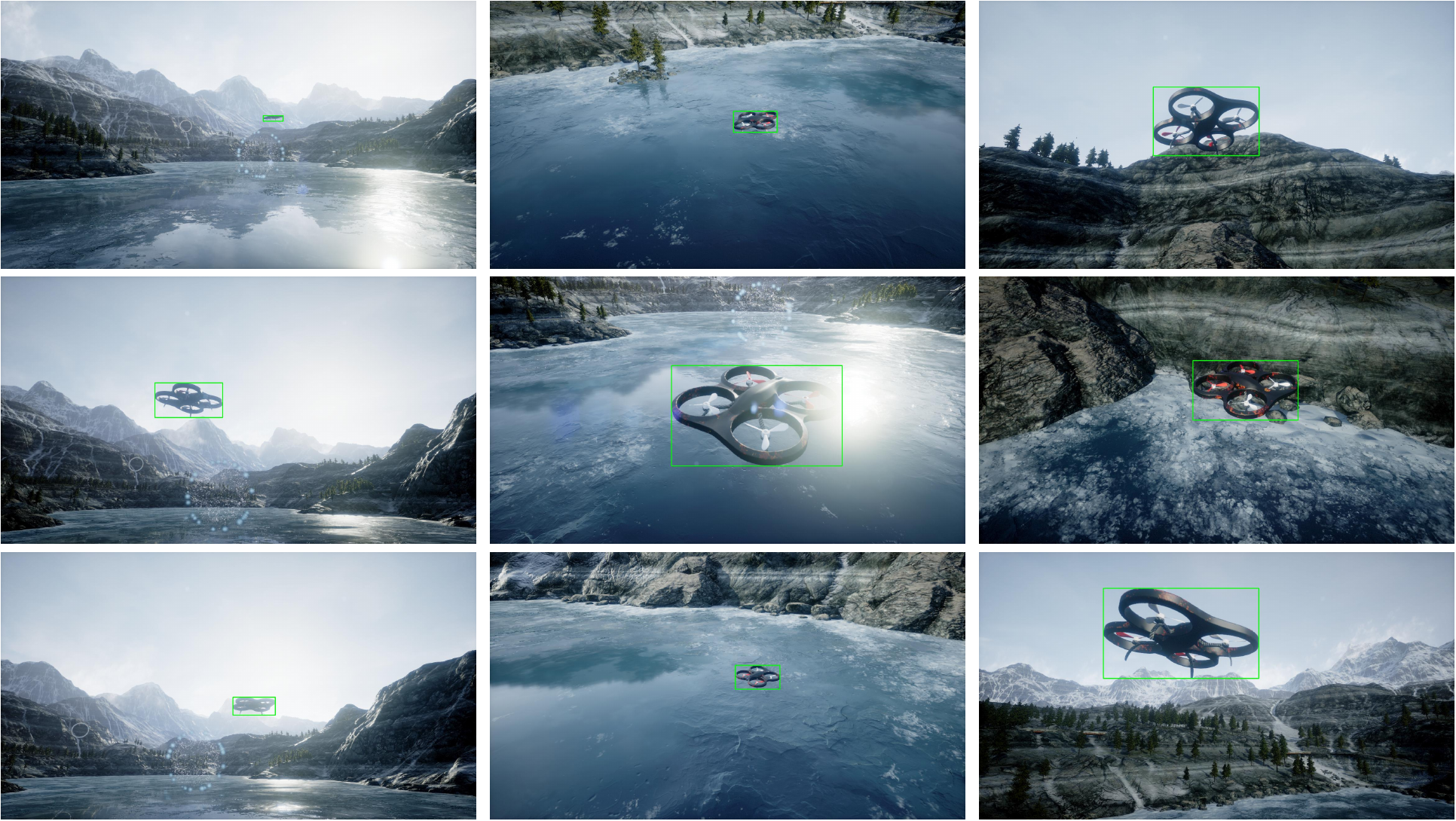}
\label{fig_airsim_dataset}
}
\caption{The setup of the AirSim simulation experiments.}
\end{figure*}

\begin{figure*}[!t]
	\centering
	\includegraphics[width=1\linewidth]{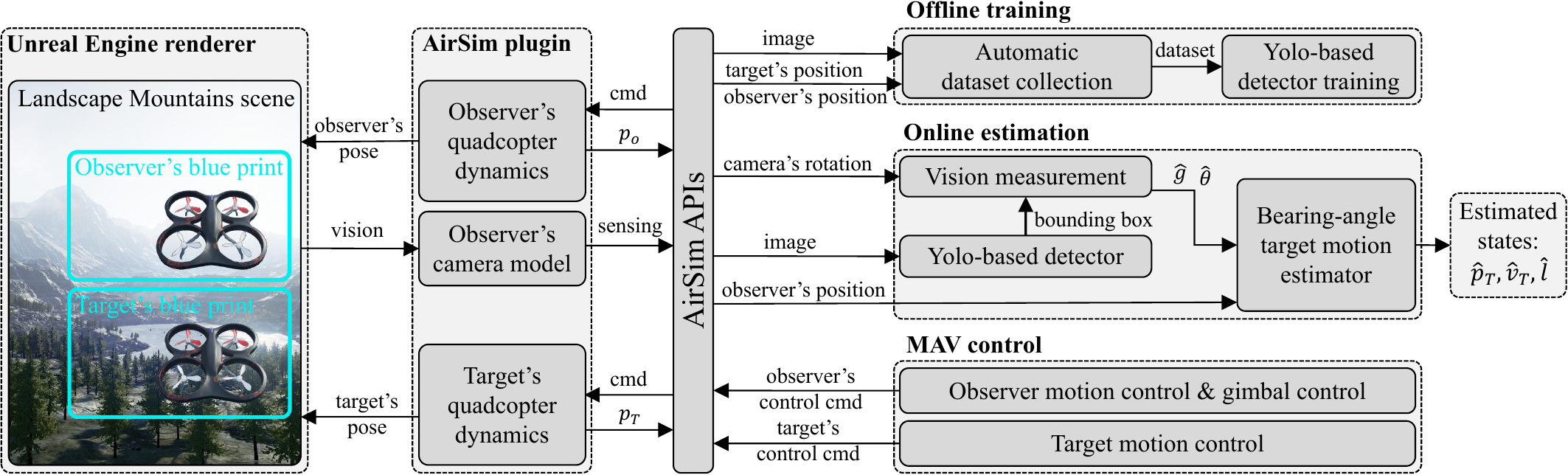}
	\caption{The software architecture of the AirSim simulation system. AirSim is a plugin for Unreal Engine. Three programmed modules (Offline training, Online estimation, and MAV control) communicate with the AirSim plugin through APIs.}
	\label{fig_box_airsim}
\end{figure*}

\subsection{Scenario 3: Approaching the target by a guidance law}
The third scenario is more complex than the first two. Here, the target moves with a constant velocity where the observer is controlled by a proportional navigation guidance (PNG) law to approach the target (Fig.~\ref{fig_matlab_3}).
During this process, both the bearing and angle vary.
This scenario is also challenging for the bearing-only approach because its observability is weak due to the fact that the lateral motion of the observer is small.
Many researchers have studied how to add extra control commands to the PNG to enhance the observability based on the bearing-only approach \citep{Song1996, Seo2015, Lee2015}.

In this simulation scenario, the target moves along a straight line with a constant velocity $\my{v}_T=[1/\sqrt{2}, 1/\sqrt{2}]^\mathrm{T}$.
The observer's velocity magnitude is constantly $3$~m/s while the velocity direction is controlled by a PNG law.
The navigation gain of the PNG law is selected as one.
The initial estimates of the target's states are the same as Scenario~1.
The simulation stops just before the observer collides with the target.

Fig.~\ref{fig_matlab_3} shows the estimation results by the bearing-only and bearing-angle approaches. As can be seen, the bearing-angle algorithm successfully converges before the collision occurs, but the bearing-only algorithm fails to estimate the target's states due to its weak observability.
This simulation example demonstrates that the bearing-angle algorithm can be used directly in the guidance scenario without extra maneuvers required by the bearing-only approach \citep{Song1996, Seo2015, Lee2015}.

\subsection{Simulation results for time-varying $\ell$}
\label{sec_matlab_varying_ell}
Although $\ell$ is assumed to be invariant, it is meaningful to challenge the proposed bearing-angle approach by considering time-varying $\ell$.
We will see through simulation examples that the bearing-angle approach is still effective when $\ell$ varies slowly. It becomes unstable when $\ell$ varies rapidly since the assumption of invariant $\ell$ is severely invalid.

Suppose that the target object has a square shape. Then, $\ell$ varies when the object is observed from different viewing angles or the object spins.
Fig.~\ref{fig_matlab_4} shows a scenario where the observer moves around the target, whose spinning speed is $2\pi$~rad/s.
The red curve in the right subfigure represents the true value of $\ell$, which varies rapidly.
As can be seen, the bearing-angle algorithm works effectively though there is a small estimation bias.
Fig.~\ref{fig_matlab_pi_8} shows a scenario where the observer moves along the bearing vector. The spinning speed of the target object is $\pi/8$~rad/s.
As can be seen, the bearing-only approach diverges due to the lack of observability. The bearing-angle algorithm can still converge since $\ell$ varies slowly.
When we further increase the spinning speed of the target, the bearing-angle algorithm will also diverge because the algorithm cannot distinguish whether the change of $\theta$ is caused by the change of $\ell$ or the change of $r$.

\section{AirSim Simulation Results}\label{sec_airsim_simulation}
In this section, we show simulation results under a more realistic setup. In particular, the simulation is based on AirSim, a simulator that can provide high-quality visual simulation \citep{Shah2017}. Nonlinear MAV dynamics and control are also considered.

\begin{figure*}[!t]
\centering
\subfloat[The target MAV hovers stationarily, while the observer MAV approaches the target MAV under the control of \eqref{eq_tracking_control}.]{
\includegraphics[width=1 \linewidth]{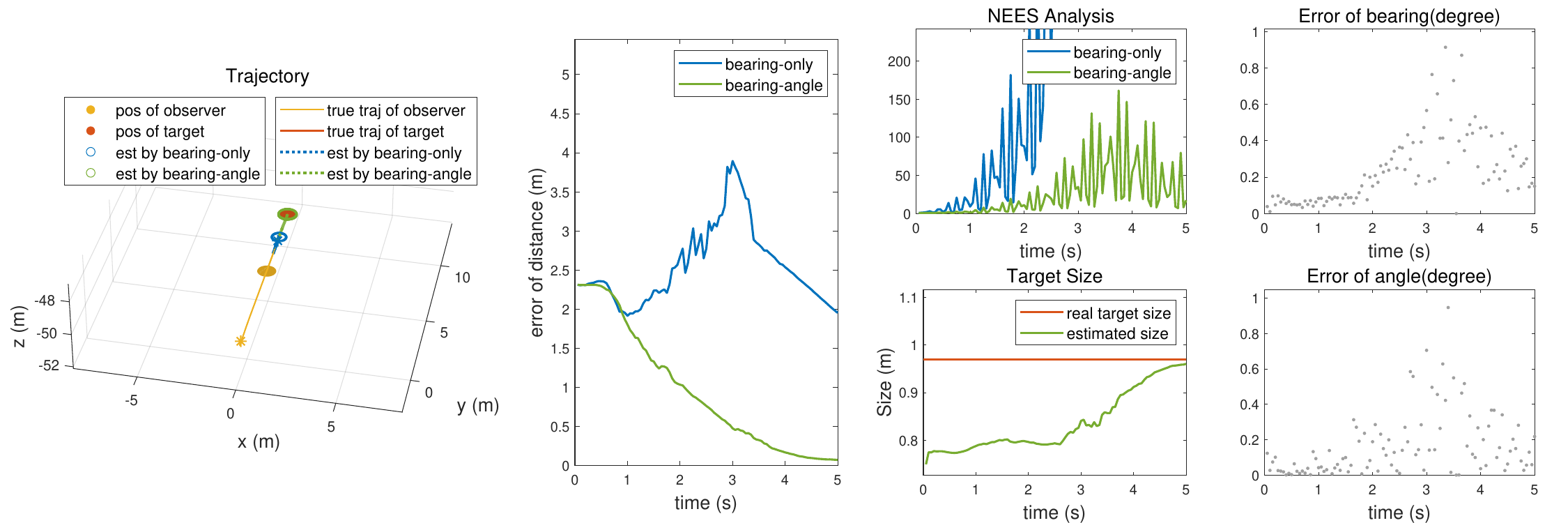}
\label{fig_airsim_1}
}
\hfill
\subfloat[The target MAV moves with a constant velocity, while the observer MAV follows the target MAV under the control of \eqref{eq_tracking_control}.]{
\includegraphics[width=1 \linewidth]{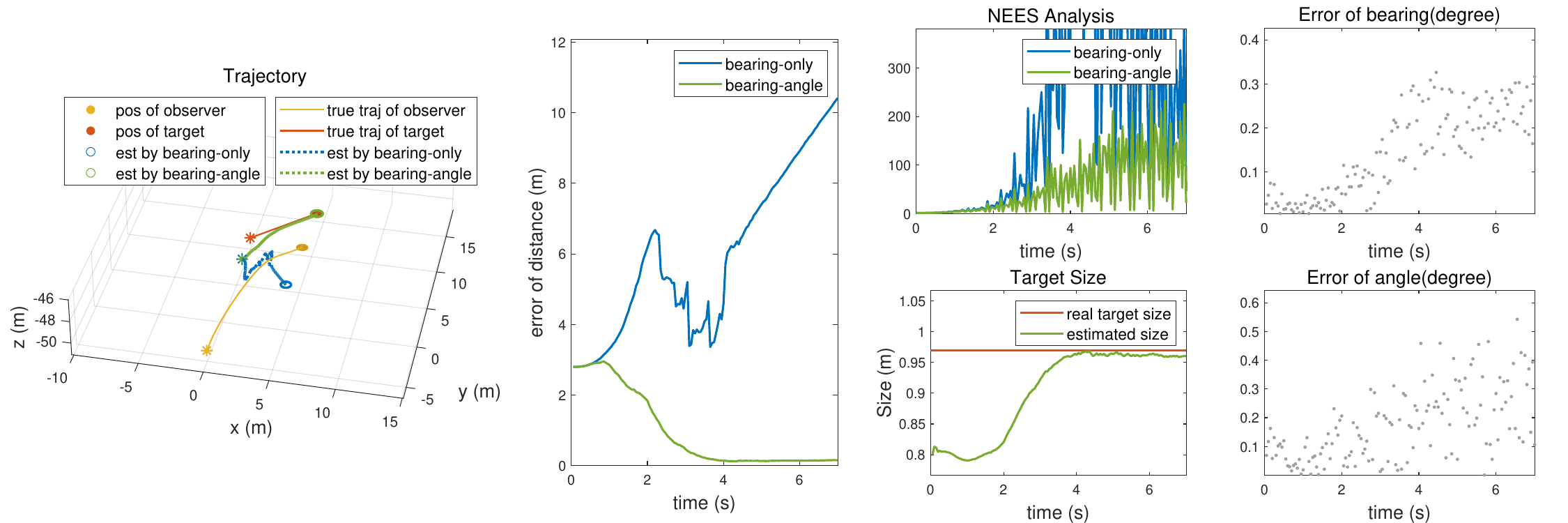}
\label{fig_airsim_2}
}
\caption{AirSim simulation results in the approaching and following scenarios.}
\label{fig_airsim}
\end{figure*}

\subsection{Simulation setup}

\begin{figure*}[!t]
\centering
\subfloat[Estimation results when $\sigma_l=10^{-4}$ and the other parameters are the same as those in Section~\ref{sec_sim_res_for_tracking}.]{
\includegraphics[width=1 \linewidth]{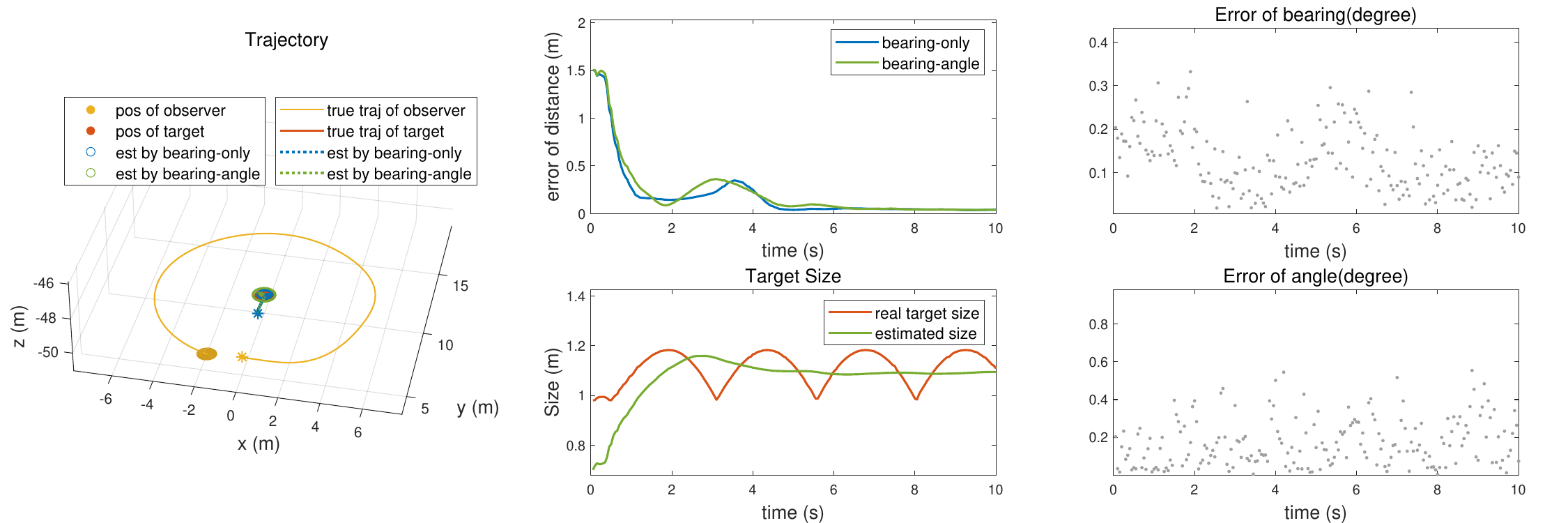}
\label{fig_airsim_6_1}
}
\hfill
\subfloat[Estimation results when $\sigma_l=0.01$ and the other parameters are the same as those in Section~\ref{sec_sim_res_for_tracking}.]{
\includegraphics[width=1 \linewidth]{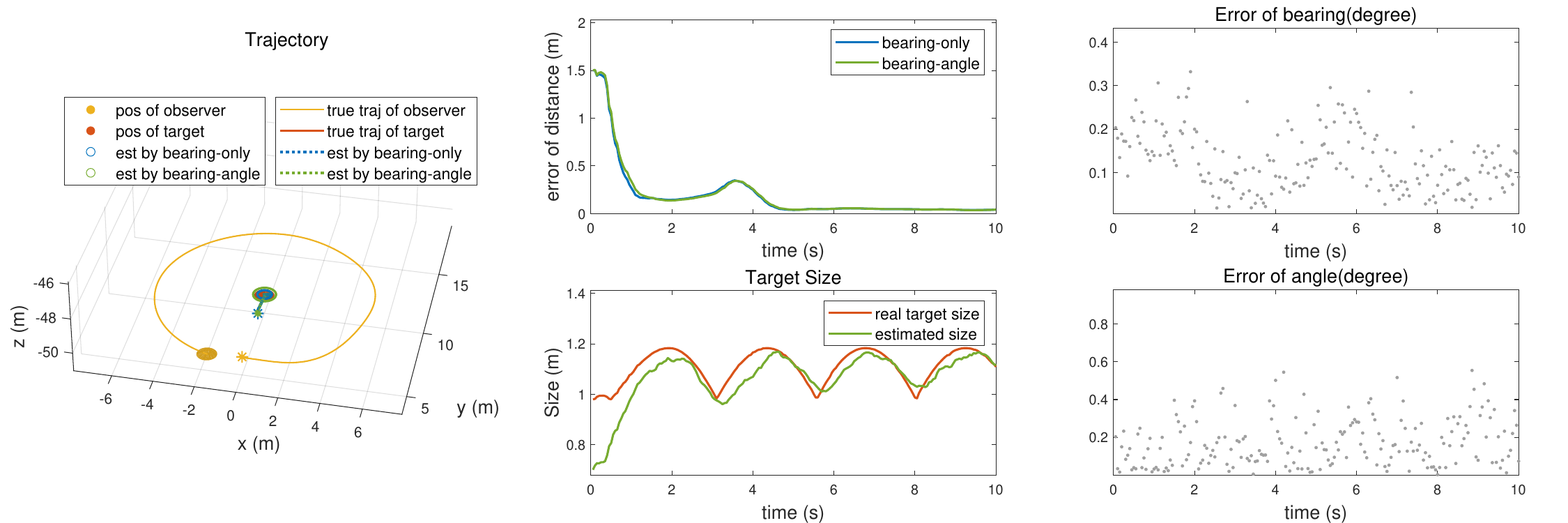}
\label{fig_airsim_6_2}
}
\caption{AirSim simulation results in the circular motion scenario where $\ell$ varies.}
\label{fig_airsim_6}
\end{figure*}

Fig.~\ref{fig_architecture_airsim} shows an AirSim simulation scenario.
As can be seen, there are two flying quadcopter MAVs. The observer MAV can capture images of the target MAV using its simulated onboard camera.
A simple gimbal camera controller is implemented so that the target MAV is always located inside the field of view of the camera.
The visual environment used in the simulation is called Landscape Mountains, which includes realistic mountains, lakes, trees, and roads. Other environments can also be used if needed.

The bearing and angle measurements are obtained from the bounding boxes generated by a Yolo-based detection algorithm.
A tiny-YOLO v4 network \citep{Bochkovskiy2020} is trained to detect the target MAV in the images. Although the visual detector can be replaced by other state-of-the-art ones, the tiny-YOLO v4 network is already sufficient to verify our proposed approach.
The architecture of the entire simulation system is shown in Fig.~\ref{fig_box_airsim}.
The system consists of the modules of automatic image dataset collection, Yolo-based target detection, gimbal camera control, nonlinear quadcopter dynamics, and quadcopter flight control.
The quadcopter dynamics and flight control used in the simulation are similar to \citep{Meier2011, Shah2017} and omitted here due to space limitation.
The quadcopter's physical size varies slightly when viewed from different directions, although it is assumed to be invariant.
All of these factors make the Airsim simulation more realistic and challenging.

\subsection{Automatic dataset collection}
To train the Yolo-based detector, we developed a module to automatically collect an image dataset.
This module has some advantages.
First, it is efficient. More than ten thousand labeled images can be collected automatically in 24 hours.
Second, it is flexible.
It can acquire images with random target's positions, random target's attitudes, random camera's view angles, and random background scenes.
These images are beneficial to achieve a good generalization ability of the detector.
Third, the image labels are of high quality. Since the ground truth of the target's image is known in the simulation, the generated bounding box is tight.
The collected dataset contains 17,000 labeled images (see Fig.~\ref{fig_airsim_dataset}).
The resolution of the images is $1536\times 864$ pixels.
The simulation system was deployed on a Dell Precision 7920 Tower Workstation with two NVIDIA Quadro GV100 graphic cards.
Since the dataset is sufficient and high-quality, the detection can achieve the accuracy of mAP=99.5\%.

\subsection{Scenario 1: Approaching and following the target}
\label{sec_sim_res_for_tracking}
We first consider the scenarios where the observer MAV approaches or follows a target MAV.
These scenarios widely exist in practical applications such as aerial target pursuit.

We show two simulation examples in Fig.~\ref{fig_airsim_1} and Fig.~\ref{fig_airsim_2}, respectively.
In both examples, the observer is controlled by a controller so that it can approach the target and maintain a desired separation. In particular, the controller is
\begin{align}
\label{eq_tracking_control}
v_o^\text{cmd}(t)&=v_T(t)+k^\text{track}\dfrac{r^2(t)-r_d^2}{r^2(t)}g(t),
\end{align}
where $v_o^\text{cmd}(t)$ is the velocity command of the observer MAV, $k^\text{track}=3$ is the control gain, and $r_d=3$ is the desired separation.
The magnitude of the observer's velocity is bounded from above by $3$~m/s.
It should be noted that \eqref{eq_tracking_control} relies on the true position and velocity of the target MAV in the simulation. Therefore, the data is collected first and then processed offline so that we can compare the performances of the bearing-only and bearing-angle approaches.

In the first example, the target MAV hovers constantly at $p_T(t_0)=[0, 10, 10]^\mathrm{T}$.
The observer MAV moves along a straight line toward the target with a decreasing velocity command.
Since the bearing of the target MAV remains the same, this example is challenging for the bearing-only approach.
As shown in Fig.~\ref{fig_airsim_1}, the bearing-only approach fails to converge while the bearing-angle approach can successfully estimate the target's motion.

In the second example, the target MAV moves with a constant velocity of $v_T=[1/\sqrt{2}, 1/\sqrt{2}, 0]^\mathrm{T}$.
The trajectory of the observer MAV under the control of \eqref{eq_tracking_control} is still close to (though not strictly) a straight line. As a result, the observability is weak by the bearing-only approach.
As shown in Fig.~\ref{fig_airsim_2}, the bearing-angle approach successfully converges while the bearing-only one fails.
It is notable that $\ell$ is invariant in the first example and varies slowly in the second example.

It is worth mentioning that the detection results used in the estimation algorithms are obtained from the Yolo-based estimator.
The ground truth obtained from AirSim is only used to calculate the errors of measurements, as shown in the right figures of Figs.~\ref{fig_airsim_1} and \ref{fig_airsim_2}.
It is not surprising that the measurement noises are not strictly Gaussian since the 2D bounding box is generated by a deep learning vision algorithm. It is noticed that the noises are inversely correlated to the observer-target range.
This is reasonable because, when the target is close to the camera and hence its image is large, the center point and the size of the bounding box usually vary for a few pixels.

The NEES values are also shown in Fig.~\ref{fig_airsim}.
As can be seen, the NEES value of the bearing-only approach diverges. The NEES value of the bearing-angle approach oscillates and converges slowly. The reasons are analyzed as follows. Compared to the Matlab-based numerical simulation, the visual measurements here are generated by deep learning algorithms, and the measurement noises are non-Gaussian. The non-Gaussian noises propagate into $P$ in \eqref{eq_nees} since the calculation of $P$ relies on noisy measurements. The noises may also cause an estimation bias that can further aggravate the NEES error. Moreover, although the system is observable in the two simulation examples, the observability is relatively weak compared to the case where the observer moves surrounding the target. As a result, the matrix $P$ may not be able to perfectly describe the estimation accuracy. These elements may jointly cause the convergence behavior of the NEES values shown in Fig.~\ref{fig_airsim}.

\subsection{Scenario 2: Circular motion and varying $\ell$}
\label{sec_sim_res_circular_scenario}

We next examine a case where $\ell$ is time-varying.
In particular, suppose a target quadcopter MAV hovers constantly at $p_T(t_0)=[0, 10, 10]^\mathrm{T}$.
The observer MAV moves on a circle centered at the target (Fig.~\ref{fig_airsim_6_1}).
Since the target quadcopter MAV has a square shape from the top view, its size $\ell$ is time-varying when viewed from side angles (see the red curves in the middle subfigure of Fig.~\ref{fig_airsim_6_1}).

We show two simulation examples in Fig.~\ref{fig_airsim_6_1} and Fig.~\ref{fig_airsim_6_2}, respectively.
The two simulation examples share the same measurement data but different values of $\sigma_\ell$.
Moreover, the other parameters are the same as those in Section~\ref{sec_sim_res_for_tracking}.

\begin{figure*}[!t]
\centering
\subfloat[Experimental setup]{
\includegraphics[width=0.48 \linewidth]{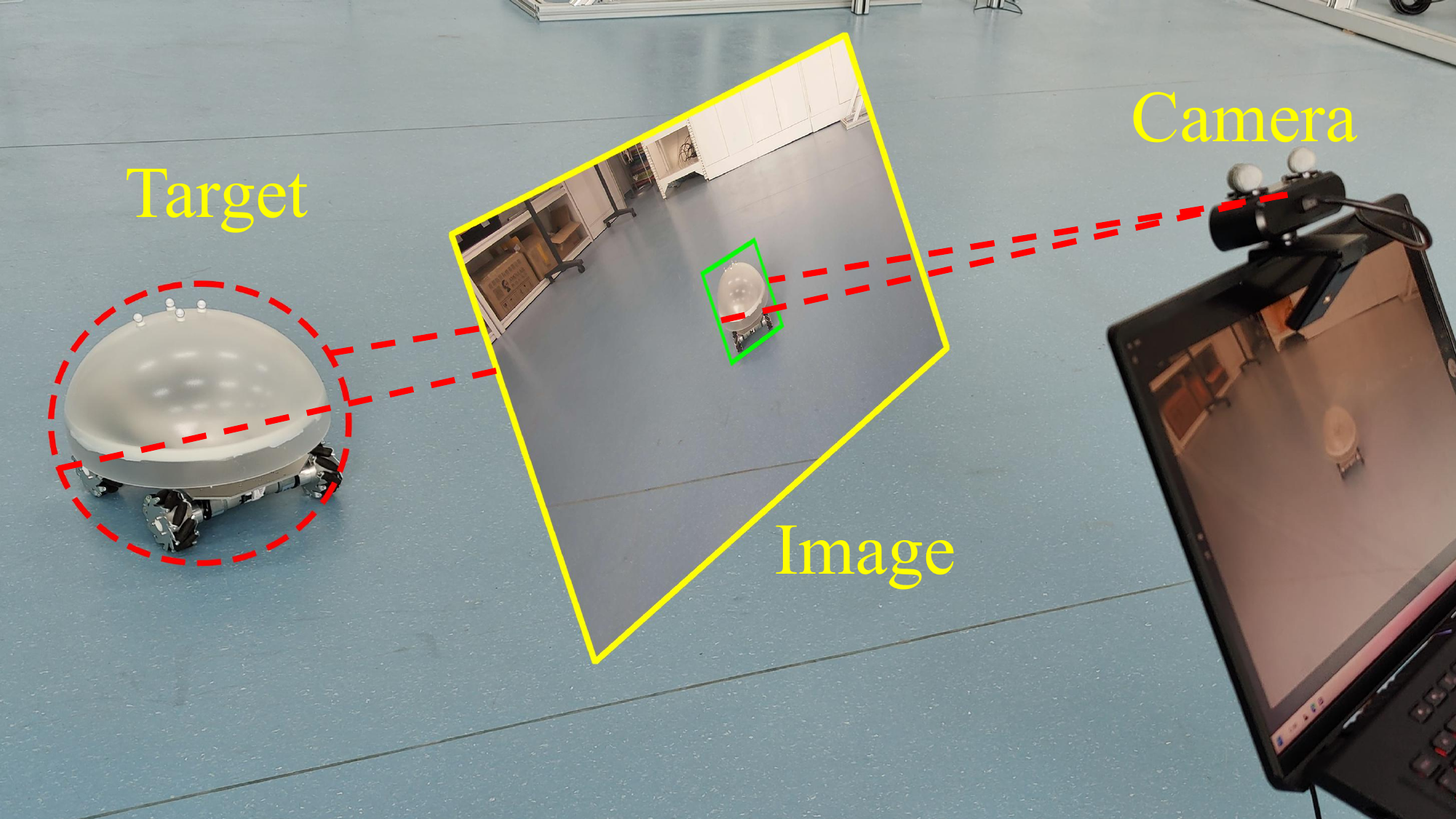}
\label{fig_architecture_indoor}
}
\subfloat[Samples in the dataset]{
\includegraphics[width=0.48 \linewidth]{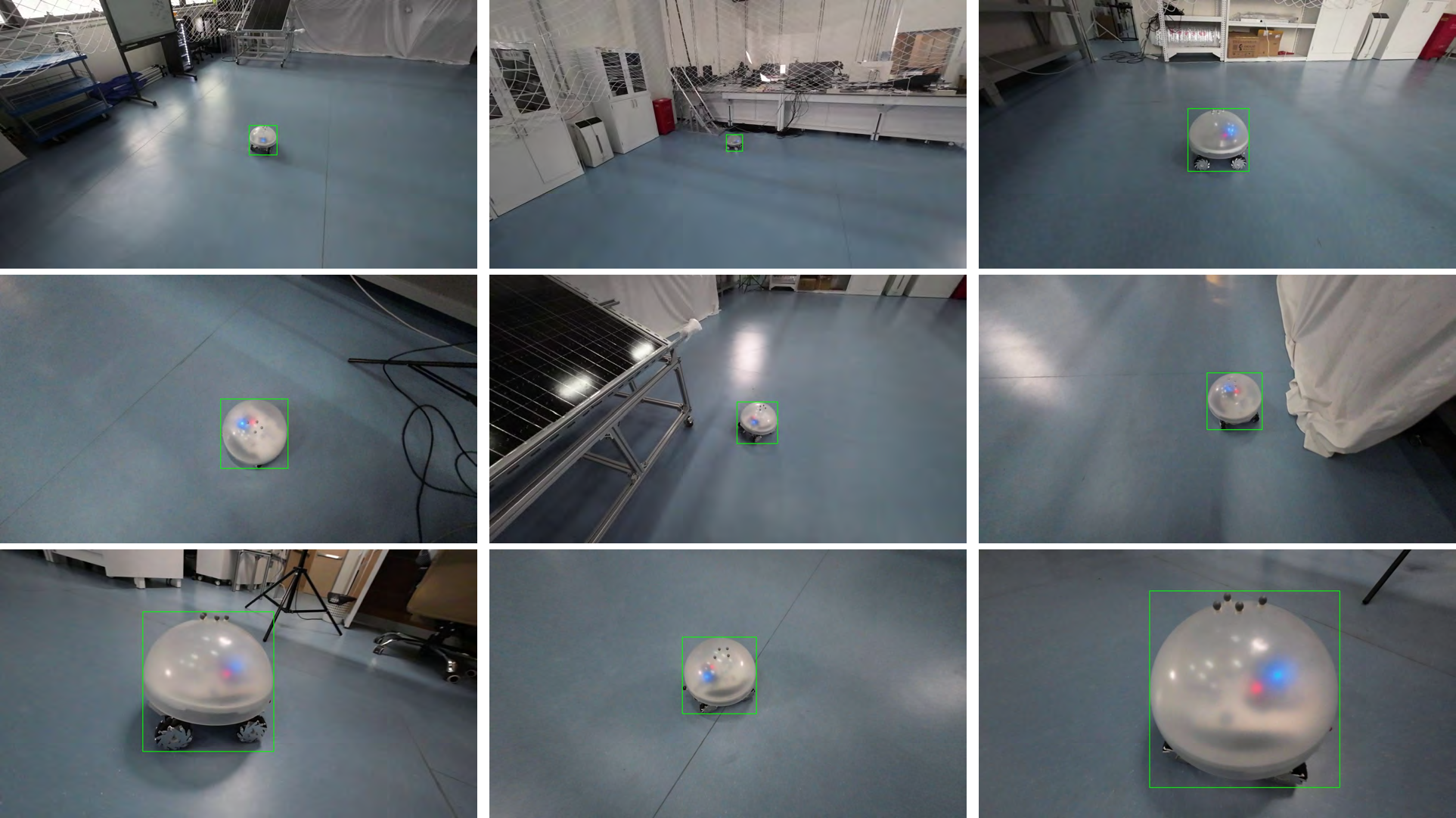}
\label{fig_car_dataset}
}
\caption{The setup of the experiments based on a hand-held camera.}
\end{figure*}

In the first simulation example, $\sigma_\ell$ is set to be a small value: $\sigma_\ell=10^{-4}$.
Its interpretation is that $\ell$ is treated as invariant during the process.
In this case, the performance of the bearing-angle approach is almost the same as the bearing-only one as shown in Fig.~\ref{fig_airsim_6_1}.
Since $\ell$ is treated to be invariant, the estimated value $\hat{\ell}$ converges to a constant which is the mean value of the time-varying $\ell$.

In the second simulation example, the value of $\sigma_\ell$ is larger than the first example: $\sigma_\ell = 0.01$.
Its interpretation is that $\ell$ is believed to be time-varying during the process.
In this case, the performance of the bearing-angle approach is still almost the same as the bearing-only one.
Moreover, since $\sigma_\ell$ is large, the bearing-angle approach can successfully estimate the true time-varying value of $\ell$.

In summary, in the case where $\ell$ varies slowly, the bearing-angle approach would degenerate to the bearing-only one.
The fundamental reason is that the extra information embedded in the angle measurement is used to estimate the time-varying $\ell$ rather than improving the observability of the target's motion.

\begin{figure*}[!t]
\centering
\subfloat[Case 1: The observer moves around the target. Both the bearing-only and bearing-angle approaches work well.]{
\includegraphics[width=1 \linewidth]{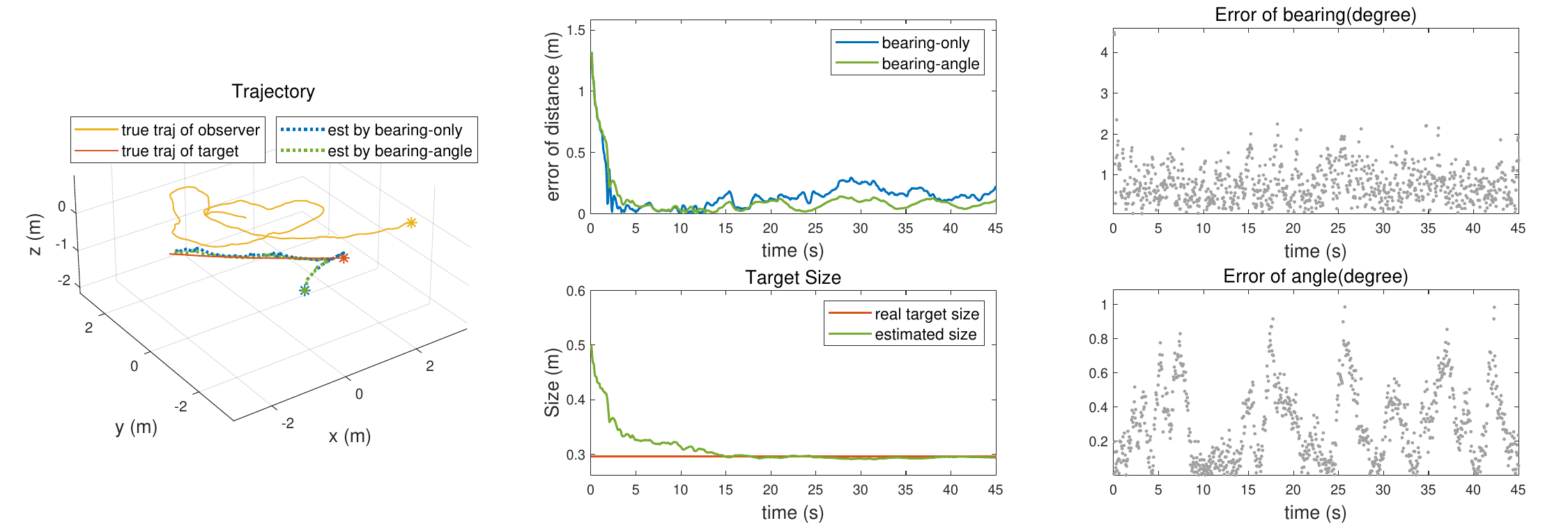}
\label{fig_indoor_9}
}
\hfill
\subfloat[Case 2: The observer moves close or far from the target periodically. The bearing-angle approach performs effectively, but the bearing-only approach works unstably.]{
\includegraphics[width=1 \linewidth]{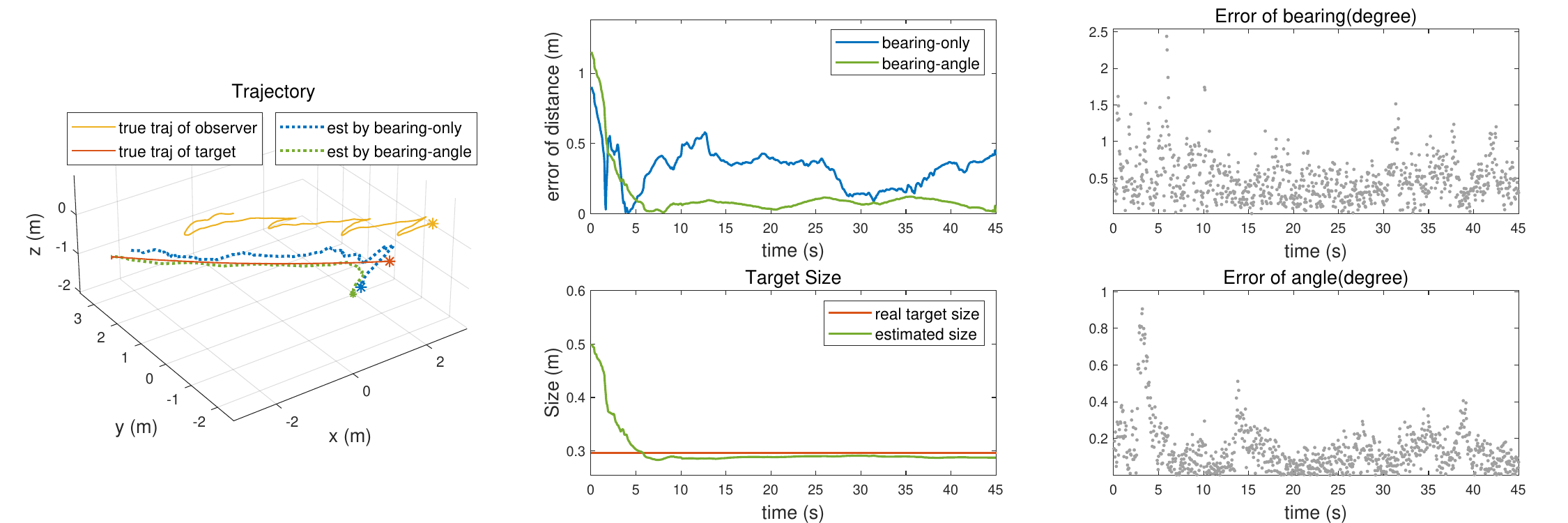}
\label{fig_indoor_6}
}
\caption{Experimental results based on a hand-held camera.}
\end{figure*}

\section{Real-World Experimental Results}\label{sec_real_world_experimental_validation}

In this section, two sets of real-world experiments are presented to further verify the effectiveness of the approach. The first is based on a hand-held camera and a ground robot.
The second is based on two quadcopter MAVs. The second experimental scenario is motivated by aerial target pursuit tasks.

\subsection{Experiment 1: Hand-held camera}

The experimental setup is shown in Fig.~\ref{fig_architecture_indoor}.
The observer is a hand-held camera (Hik Vision DS-E14S) connected to a laptop. The camera's intrinsic parameters are calibrated beforehand.
The robot built on Mecanum wheels can move in any direction on the ground under velocity control.
The ground truth of the states of the camera and the robot are provided by a Vicon indoor motion capture system.
The key experimental specifications are listed in Table~\ref{table_indoor_hardware}.

A dataset of 5,514 images was collected and used to train a tiny-YOLO v4 network to detect the target robot (see Fig.~\ref{fig_car_dataset}).
The detection precision of the trained network is mAP=99.8\%.
In the experiment, the target robot is commanded to move with a constant velocity.
In the meantime, a person holding the camera moves along some trajectories.
Two different cases are studied. In both of the cases, the target robot moves with a constant velocity $v_T=[-0.1, 0.1 ,0]^\mathrm{T}$.
The noises of the measurements are calculated based on the ground truth provided by the Vicon system. The noises are shown in the right subfigures of Fig.~\ref{fig_indoor_9} and Fig.~\ref{fig_indoor_6}.

In the first case, the camera is held about 1.5 meters above the ground and moves around the target robot. In this case, the bearing vector varies sufficiently and hence the observability conditions for the bearing-only and bearing-angle approaches are both well satisfied. As shown in Fig.~\ref{fig_indoor_9}, both approaches perform well in this case while the bearing-angle approach performs slightly better than the bearing-only one.

In the second case, the camera moves along the trajectory of the robot by getting close or far from it periodically.
In this case, the angle varies significantly, but the bearing does not.
Without surprise, the bearing-only approach performs poorly in this case due to weak observability  (Fig.~\ref{fig_indoor_6}).
By contrast, the bearing-angle approach can perform stably due to its enhanced observability.
\begin{table}
\begin{center}
\caption{Key specifications of the \emph{indoor} hardware system.}
\label{table_indoor_hardware}
\begin{tabular}{l|lll}
\hline
 & Parameter & Value & Unit \\
\hline
\multirow{2}*{Camera} & Resolution & 640$\times$ 480 & pixel\\
~ & Max frequency & 30 &fps\\
\hline
\multirow{2}*{Robot} & Max speed & 1& m/s \\
~ & Diameter size & 295 & mm\\
\hline
\multirow{2}*{Vicon} & Localization accuracy & 1 & mm \\
~ & Max frequency & 100 & Hz\\
\hline
\end{tabular}
\end{center}
\end{table}

\subsection{Experiment 2: MAV-following-MAV}
\label{sec_mav_following_mav}

\begin{table}
\begin{center}
\caption{Key specifications of the \emph{outdoor} hardware system.}
\label{table_M300}
\begin{tabular}{c|lll}
\hline
 & Parameter & Value & Unit \\
\hline
\multirow{5}{*}{\makecell[c]{M300 \\quadcopter}} & Diagonal size & 895 & mm\\
~&Total mass & 7.4 & kg \\
~&Max pitch/roll & 30 & degree \\
~&Max flight time & 30 & minutes\\
\hline
\multirow{2}{*}{RTK} & Accuracy & 1 & cm \\
~& Max frequency & 10 & Hz\\
\hline
\multirow{3}{*}{\makecell[c]{H20 \\ gimbal \& \\ camera} } & Resolution &1920$\times$1080  & pixel\\
~& Frequency & 15 & Hz \\
~ & Max angular rate & 180 & deg/s\\
\hline
\end{tabular}
\end{center}
\end{table}

\begin{figure*}[!t]
\centering
\subfloat[Hardware platforms]{
\includegraphics[width=0.48 \linewidth]{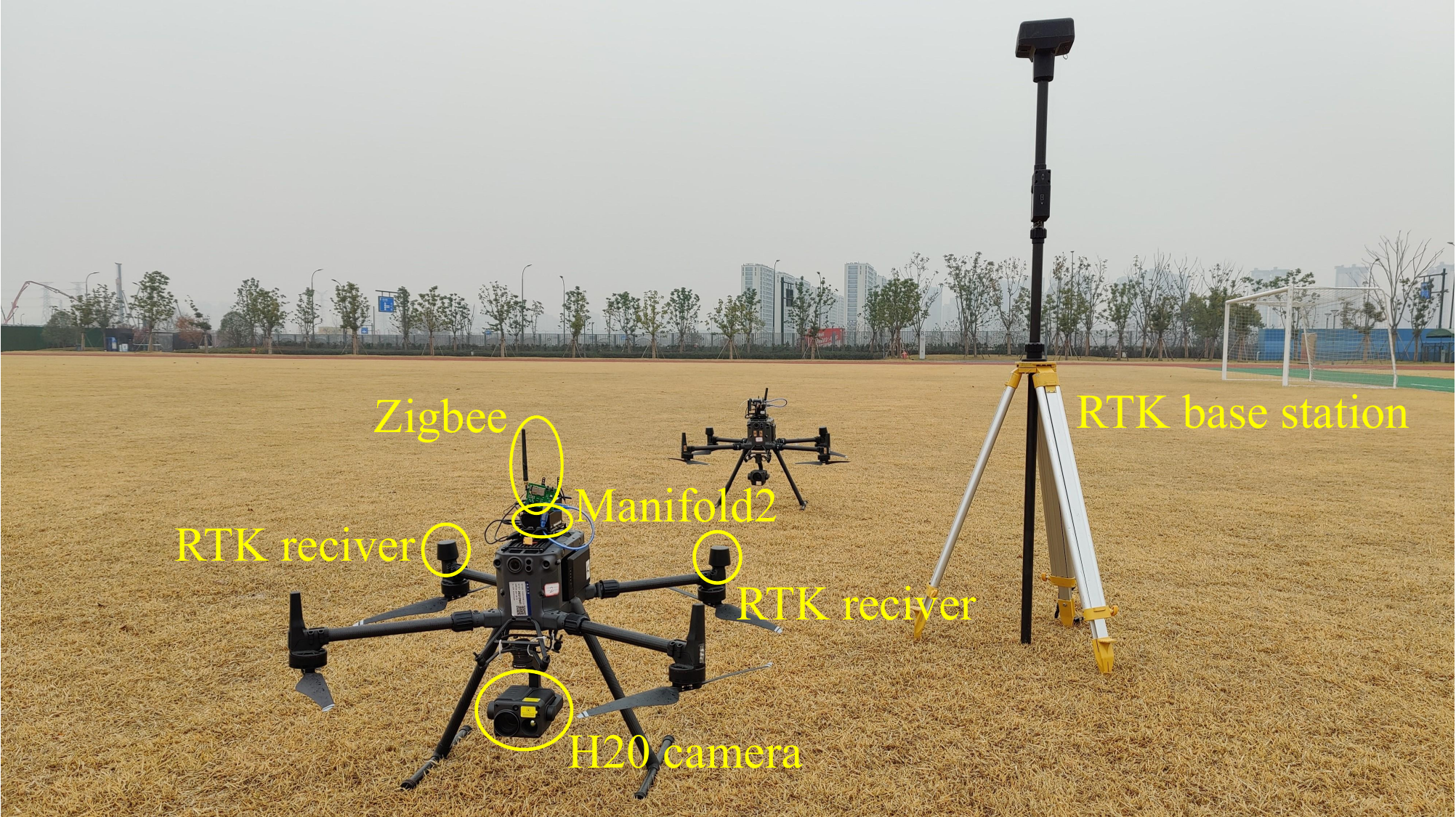}
\label{fig_M300}
}
\subfloat[Samples of the images in the dataset]{
\includegraphics[width=0.48\linewidth]{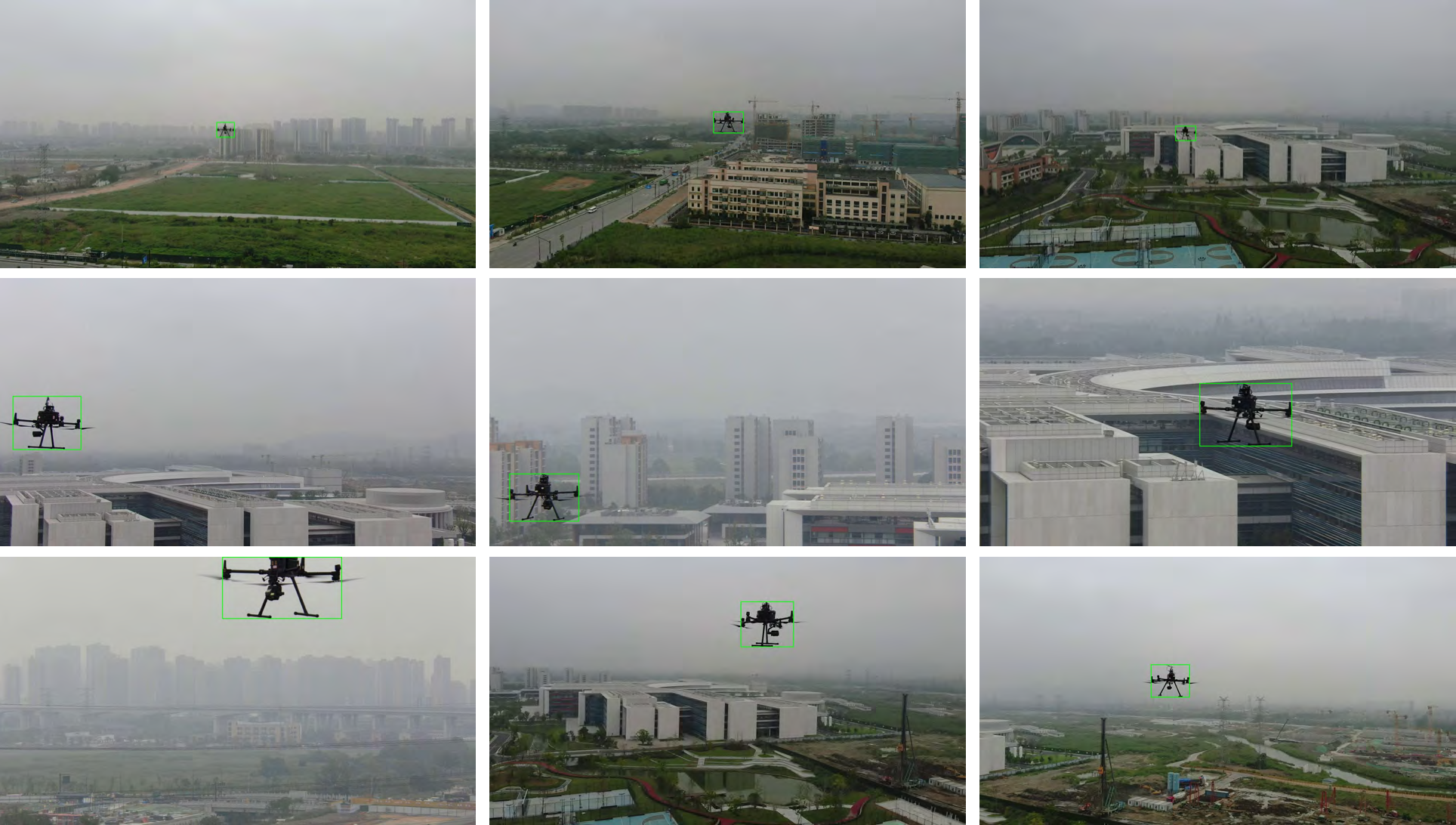}
\label{fig_M300_dataset}
}
\hfill
\subfloat[System architecture]{
\includegraphics[width=0.48 \linewidth]{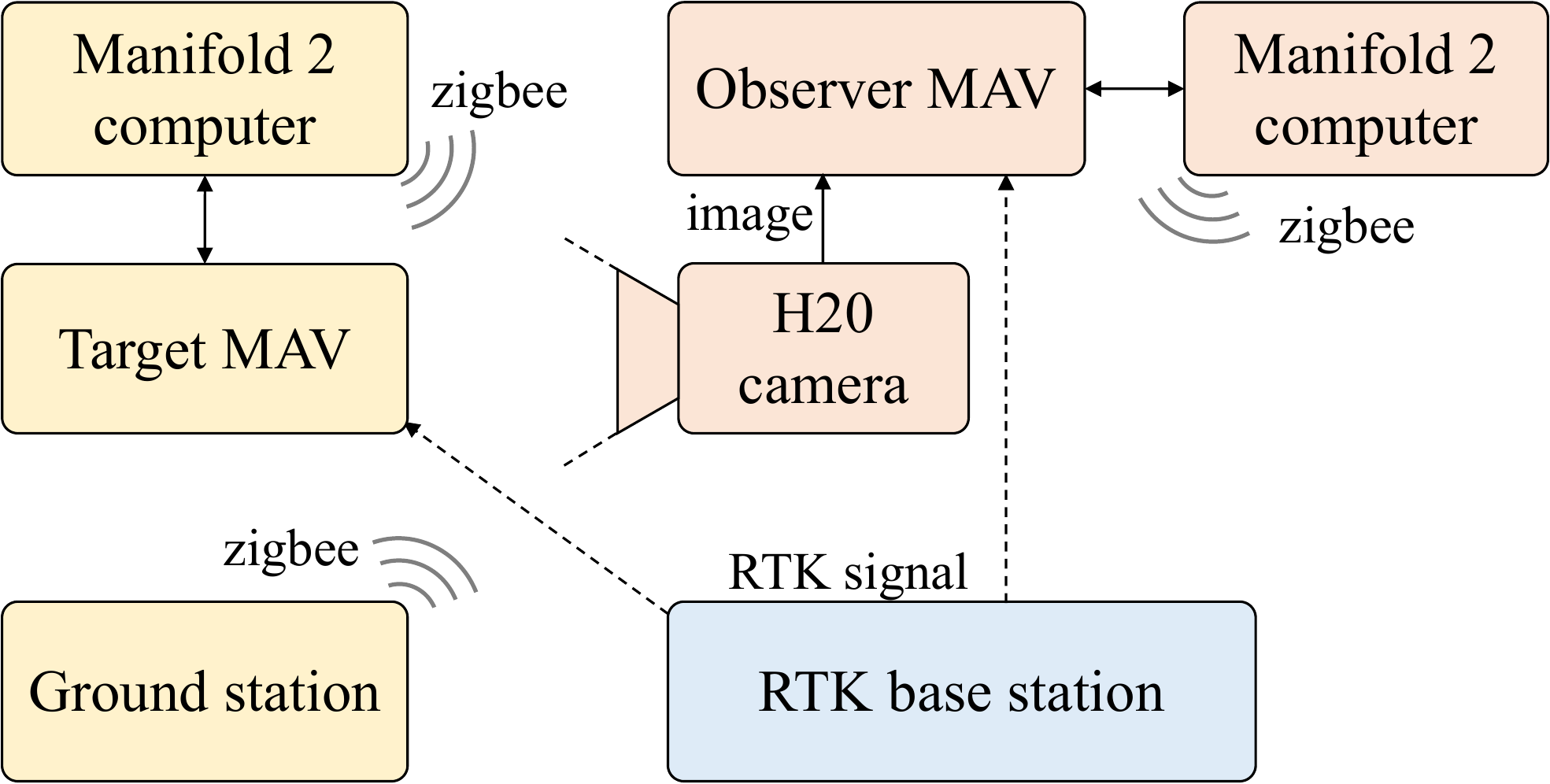}
\label{fig_outdoor_hardware}
}
\caption{The setup of the MAV-following-MAV experiment.}
\end{figure*}

\begin{figure*}[!t]
	\centering
	\includegraphics[width=1\linewidth]{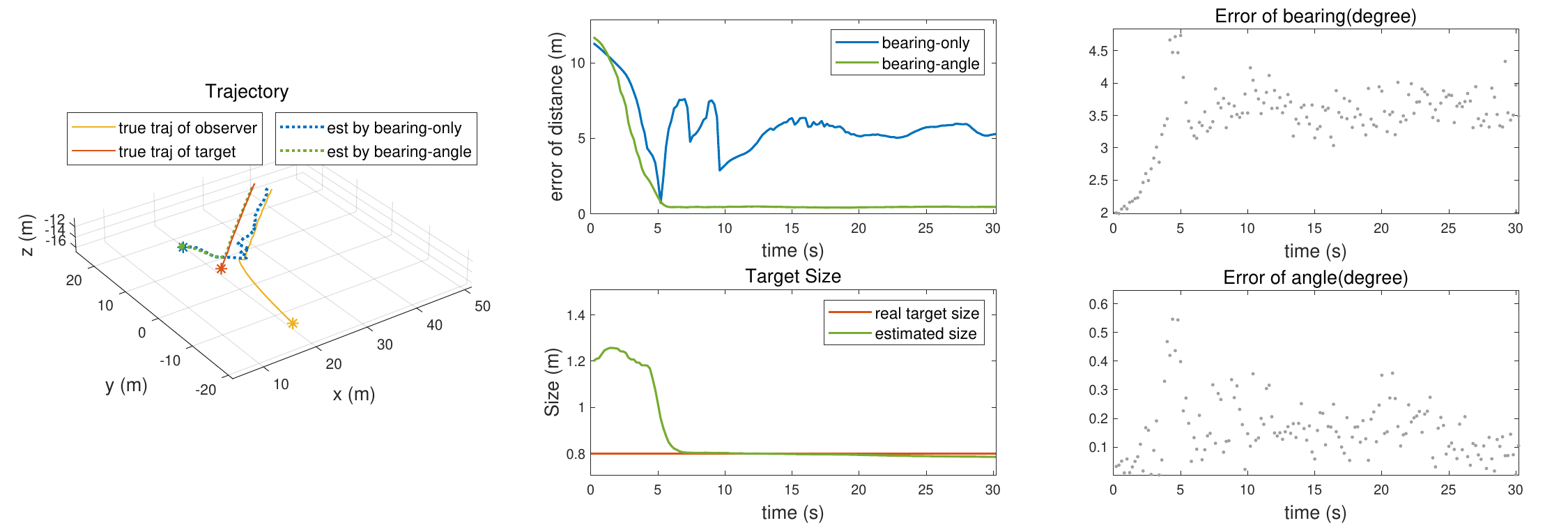}
	\caption{The results of the MAV-following-MAV experiments. The bearing-angle approach performs effectively, but the bearing-only approach works unstably.}
	\label{fig_outdoor_1}
\end{figure*}

Two MAV platforms were developed based on DJI M300 quadcopters (Fig.~\ref{fig_M300}).
The MAV platforms are equipped with RTK GPS modules for accurate self-localization, an H20 camera for visual detection, a Manifold 2G onboard computer for onboard flight control, and a Zigbee module for wireless communication.
Some key specifications of the MAV platforms are listed in Table~\ref{table_M300}.
The structure of the hardware perception and communication system is illustrated in Fig.~\ref{fig_outdoor_hardware}.
The target MAV is also equipped with an RTK GPS module, whose measurements are used as the ground truth to calculate the noises of the visual measurements. The noises are shown in the right subfigure of Fig.~\ref{fig_outdoor_1}.

The experiment consists of two stages: data acquisition and offline data processing.
In the data acquisition stage, the target MAV is commanded to fly with a constant velocity, and the observer MAV is automatically controlled to follow the target MAV to maintain a constant distance from the target.
More specifically, the procedure of the flight experiment is as follows. Initially, the observer MAV is placed about 20 meters behind the target MAV on the ground. Then, the two MAVs take off and fly to the same specified height automatically upon a takeoff command sent from the ground control station.
After they have reached the desired height, all deployed algorithms are activated.
Then, the target MAV moves with a constant velocity of $v_T=[1/\sqrt{2}, 1/\sqrt{2}, 0]^\mathrm{T}$. The observer MAV approaches the target by the controller in \eqref{eq_tracking_control}. It takes the observer MAV about eight seconds to reach the desired relative distance.
Then, the two MAVs fly with the same velocity and remain relatively stationary for another 20 seconds.
Finally, the ground station sends a stop command and the two MAVs return and land automatically.

During the flight, the gimbal camera of the observer MAV is automatically controlled so that the target MAV is maintained in the field of view.
It is noted that the control of the gimbal camera and the observer MAV is not based on the visual detection results. Instead, the control is based on the measurements provided by the RKT GPS and inter-MAV wireless communication. In this way, we can analyze the image and flight data offline and compare the performance of the two approaches of bearing-angle and bearing-only.
The acquired images and flight data are saved on the onboard computer during the flight. A dataset of 5,341 images was collected (Fig.~\ref{fig_M300_dataset}) and used to train a tiny-YOLO v4 network.
The detection precision of the trained network reaches mAP=99.8\%.

The experimental results are shown in Fig.~\ref{fig_outdoor_1}.
As can be seen, the bearing-angle approach performs well. By contrast, the bearing-only approach only works well before the observer MAV reached the desired position relative to the target MAV. That is because the bearing varies significantly during this process due to the fluctuation of the observer's motion caused by the flight control. However, the bearing-only approach diverges quickly thereafter when the bearing stops varying significantly.

\section{Conclusion}\label{section_conclusion}
Motivated by the limitation of the existing bearing-only approach, this paper proposed and analyzed a novel bearing-angle approach for vision-based target motion estimation. We showed that the observability by the bearing-angle approach is significantly enhanced compared to the bearing-only one.
As a result, the requirement of the observer's extra motion for observability enhancement can be significantly relaxed.
As we showed in various experiments, the bearing-angle approach can successfully estimate the target's motion in many scenarios where the bearing-only approach fails.
The enhanced observability of the bearing-angle approach comes with no additional cost since almost all vision detection algorithms can generate bounding boxes.
One assumption adopted in the bearing-angle approach is that the target's physical size is invariant to different viewing angles. Although this assumption is approximately valid in many tasks as demonstrated in this paper, it is meaningful to study how to relax or remove this assumption in the future.


\section*{Declaration of conflicting interests}
The author(s) declared no potential conflicts of interest with respect to the research, authorship, and/or publication of this article.

\section*{Funding}
The author(s) disclosed receipt of the following financial support for the research, authorship, and/or publication of this artical: This work was supported by the Hangzhou Key Technology Research and Development Program (Grant 20212013B09), and the Research Center for Industries of the Future at Westlake University (Grant WU2022C027).

\bibliographystyle{SageH}
\bibliography{paper}

\end{document}